\documentclass{article}

\usepackage[preprint]{neurips_2022}
\usepackage{header}

\usepackage[utf8]{inputenc} 
\usepackage[T1]{fontenc}    
\usepackage{hyperref}       
\usepackage{url}            
\usepackage{booktabs}       
\usepackage{amsfonts}       
\usepackage{nicefrac}       
\usepackage{microtype}      
\usepackage{xcolor}         

\title{Generalization for Multiclass Classification with Overparameterized Linear Models}

\author{%
  Vignesh Subramanian \\
  Department of Electrical Engineering and Computer Sciences \\
  University of California Berkeley \\
  Berkeley, CA-94720, USA \\
  \texttt{vignesh.subramanian@eecs.berkeley.edu} \\
  \And
  Rahul Arya \\
  Department of Electrical Engineering and Computer Sciences \\
  University of California Berkeley \\
  Berkeley, CA-94720, USA \\
  \texttt{rahularya@berkeley.edu} \\
  \And
  Anant Sahai \\
  Department of Electrical Engineering and Computer Sciences \\
  University of California Berkeley \\
  Berkeley, CA-94720, USA \\
  \texttt{sahai@eecs.berkeley.edu} \\
}

\begin{document}

\maketitle

\newcommand{\survivalbinary}{\ensuremath{\mathsf{SU_b}}}
\newcommand{\Amk}{\ensuremath{\mathbf{A}_{-\truek}}}
\newcommand{\Amkinv}{\ensuremath{\Amk^{-1}}}
\newcommand{\tr}{\ensuremath{\mathsf{tr}}}
\newcommand{\truek}{\ensuremath{\tau}}
\newcommand{\tran}{\ensuremath{^\top}}
\newcommand{\cdel}{\ensuremath{c_\Delta}}
\newcommand{\tdel}{\ensuremath{\Delta}}
\newcommand{\pdel}{\ensuremath{p_\Delta}}
\newcommand{\EE}{\ensuremath{\mathbb{E}}}
\newcommand{\PR}{\ensuremath{\mathbb{P}}}
\newcommand{\ztk}{\ensuremath{\mathbf{z}_\truek}}
\newcommand{\ytk}{\ensuremath{\mathbf{y}_\truek}}
\newcommand{\lambdatd}{\ensuremath{\tilde{\lambda}}}
\newcommand{\ztki}{\ensuremath{z_{\truek,i}}}
\newcommand{\yi}{\ensuremath{y_i}}
\newcommand{\Mbold}{\ensuremath{\mathbf{M}}}
\newcommand{\z}{\ensuremath{\mathbf{z}}}
\newcommand{\sutk}{\ensuremath{\mathsf{SU}(\truek)}}
\newcommand{\lambdatk}{\ensuremath{\lambda_{\truek}}}
\newcommand{\cn}{\ensuremath{\mathsf{CN}}}
\newcommand{\Cbold}{\ensuremath{\mathbf{C}}}
\newcommand{\Sigmabold}{\ensuremath{\boldsymbol{\Sigma}}}
\newcommand{\lambdabold}{\ensuremath{\boldsymbol{\lambda}}}
\newcommand{\Abold}{\ensuremath{\mathbf{A}}}
\newcommand{\zj}{\ensuremath{\z_j}}
\renewcommand{\a}{\ensuremath{\alpha}}
\renewcommand{\b}{\ensuremath{\beta}}
\newcommand{\za}{\ensuremath{\mathbf{z}_\a}}
\newcommand{\zai}{\ensuremath{z_{\a, i}}}
\newcommand{\zb}{\ensuremath{\mathbf{z}_\b}}
\newcommand{\ya}{\ensuremath{\mathbf{y}_\a}}
\newcommand{\yaohi}{\ensuremath{y^{oh}_{\a,i}}}
\newcommand{\yaoh}{\ensuremath{\mathbf{y}^{oh}_{\a}}}
\newcommand{\yb}{\ensuremath{\mathbf{y}_\b}}
\newcommand{\ybohi}{\ensuremath{y^{oh}_{\b,i}}}
\newcommand{\yboh}{\ensuremath{\mathbf{y}^{oh}_{\b}}}
\newcommand{\dely}{\ensuremath{\boldsymbol{\Delta}y}}
\newcommand{\mubar}{\ensuremath{\bar{\mu}}}
\newcommand{\delmu}{\ensuremath{\Delta_{\mu}}}
\newcommand{\indicator}{\ensuremath{\mathbb{I}}}
\newcommand{\nc}{\ensuremath{k}}
\newcommand{\Ainv}{\ensuremath{\mathbf{A}^{-1}}}
\newcommand{\Ainvmu}{\ensuremath{\bar{\mathbf{A}}_{inv}}}
\newcommand{\Ainvdel}{\ensuremath{\boldsymbol{\Delta}{A}_{inv}}}
\newcommand{\fhat}{\widehat{f}}
\newcommand{\ghat}{\widehat{g}}
\newcommand{\hhat}{\widehat{h}}
\newcommand{\mhat}{\widehat{m}}
\newcommand{\maxevent}{\mathcal{E}}
\newcommand{\nonmaxevent}{\mathcal{E}^c}
\newcommand{\todo}[1]{\textcolor{red}{#1}}
\newcommand{\rahul}[1]{{\textcolor[rgb]{1,0,1}{\textsc{[Rahul: {#1}]}}}}
\newcommand{\prof}[1]{{\textcolor[rgb]{0,0,1}{\textsc{[Prof. Sahai: {#1}]}}}}
\newcommand{\vignesh}[1]{{\textcolor[rgb]{0,0.5,1}{\textsc{[Vignesh: {#1}]}}}}
\newcommand{\xtest}{\mathbf{x}_{test}}
\newcommand{\xwtest}{\mathbf{x}^w_{test}}
\newcommand{\Eerr}{\mathcal{E}_{err}}
\newcommand{\Mtypical}{\mathcal{M}}
\newcommand{\Mii}{M_{ii}}
\newcommand{\Mtii}{\Mtypical_{ii}}
\newcommand{\delyi}{\Delta y[i]}
\newcommand{\zji}{z_j[i]}
\newcommand{\Ainvbold}{\mathbf{A}^{-1}}
\newcommand{\Xw}{\mathbf{X}^w}
\newcommand{\In}{\mathbf{I}_n}

\newcommand{\app}[1]{Appendix \ref{app:#1}}
\newcommand{\apps}[1]{#1}

\begin{abstract}
Via an overparameterized linear model with Gaussian features, we provide conditions for good generalization for multiclass classification of minimum-norm interpolating solutions in an asymptotic setting where both the number of underlying features and the number of classes scale with the number of training points. The survival/contamination analysis framework for understanding the behavior of overparameterized learning problems is adapted to this setting, revealing that multiclass classification qualitatively behaves like binary classification in that, as long as there are not too many classes (made precise in the paper), it is possible to generalize well even in some settings where the corresponding regression tasks would not generalize. Besides various technical challenges, it turns out that the key difference from the binary classification setting is that there are relatively fewer positive training examples of each class in the multiclass setting as the number of classes increases, making the multiclass problem ``harder'' than the binary one.  


\end{abstract}


\section{Introduction}
Multiclass classification on standardized datasets is where the current deep-learning revolution really made the community take notice with previously unattainable levels of performance. 
Contemporary systems have demonstrated tremendous success at these tasks, typically using gigantic models with parameters that vastly exceed the (also large) number of data points used to train these models. In defiance of traditional statistical wisdom regarding overfitting, these big models can be trained to achieve zero training error even with noisy labels, but still generalize well in practice \citep{reg:zhang2016understanding, reg:geiger2019jamming}. 

To better understand this empirical phenomenon, one line of work uses appropriate high-dimensional linear models for regression problems to show how benign fitting of noise in training data is possible \citep{reg:hastie2019surprises, reg:mei2019generalization, reg:bartlett2020benign, reg:belkin2020two, reg:muthukumar2020harmless}. Essentially, the model must have enough "non-preferred" degrees of freedom to be able to absorb the training noise without contaminating predictions by too much. Simultaneously, there has to be enough of a preference for degrees of freedom that can capture the true pattern to enable it to survive the learning procedure and be well represented in the final learned model. 

A subsequent line of work studies binary classification \citep{binary:Muth20, binary_classification:chatterji2021finite, binary_classification:wang2021benign} and shows that binary classification can generalize well beyond what can be proved by classical margin-based bounds \citep{margin:bartlett2002rademacher} and there exist regimes where binary classification can even succeed in generalizing where regression fails --- less preference is required for the degrees of freedom that capture the true pattern \citep{binary:Muth20}. 
Very recently, the generalization of multiclass classification in similar models was studied in \citet{multi_class_theory:Wang21} but the analysis was limited to a fixed finite number of classes. 
In practice, we see that larger datasets often come with more classes and are tackled with even bigger models and so it is important to see what happens to generalization when everything scales together. 
To have a crisply understandable approach that allows everything to scale, this paper also adopts the bi-level covariance model with Gaussian features that is used in \citet{reg:muthukumar2020harmless, binary:Muth20, reg:wang2021tight, multi_class_theory:Wang21}.

To understand classification, we must understand the role of training loss functions in determining what is learned. 
Empirical evidence shows that least-squares can yield classification performance competitive to cross-entropy minimization \citep{algo_comparison:Rifkin04, nets_loss_function:Hui20, nets_loss_function:Bos20}. 
\citet{binary:Muth20, binary_classification:hsu2021proliferation} show that indeed with sufficient overparameterization, the support vector machine (SVM) solution, which also arises from minimizing the logistic loss using gradient descent \citep{implicit_bias:srebro, implicit_bias:ji2019}, is identical to that obtained by the minimum-norm interpolation (MNI) of binary labels --- what would be obtained by gradient descent while minimizing the squared loss. 
A similar equivalence\footnote{For an interesting alternative perspective on this equivalence as an indication of a potential bug instead of as a promising feature, see \citet{binary_classification:shamir2022implicit}.} holds for different variations of multiclass SVMs and the MNI of one-hot-encoded labels \citep{multi_class_theory:Wang21}. 
Consequently, this paper focuses on the MNI approach to overparameterized learning for multiclass classification.

\section{Our contributions}

Our study provides an asymptotic analysis of the error of the minimum-norm interpolating classifier for the multiclass classification problem with weighted Gaussian features. We consider an overparameterized setting using a bi-level feature weighting model where the number of features, classes, favored features, and the feature weights themselves all scale with the number of training points. Under this model, Theorem~\ref{theorem:regimes} provides sufficient conditions for good generalization in the form of a region in which as the number of training points increase, the number of classes grows slowly enough, the total number of features (i.e.~level of overparameterization) grows fast enough, the number of favored features grows slowly enough, and the amount of favoring of those favored features is sufficient to allow for asymptotic generalization. We assume that our labels are generated noiselessly based on which of the first $k$ features is the largest.\footnote{This assumption is without loss of generality for the bi-level model as long as the classes are defined by orthogonal directions as in \citet{multi_class_theory:Wang21}.} 

To prove our main result, Theorem~\ref{theorem:regimes}, we present a novel typicality-style argument featuring the feature 
margin (gap between the largest and second-largest feature) for computing sufficient conditions for correct classification utilizing the signal-processing inspired concepts of survival and contamination from \citet{reg:muthukumar2020harmless, binary:Muth20} and leveraging the random-matrix analysis tools sharpened in \citet{reg:bartlett2020benign}.
The key is analyzing what happens with multiclass training data where there are relatively fewer positive examples of each class, and where the training data for a particular class is not independent of the features corresponding to other classes. 
The analysis shows that as a result of having fewer positive exemplars for a class relative to the total size of the training data, the survival drops by a factor of $k$ (the number of classes), while the contamination only drops by a factor of $\sqrt{k}$. 
As in binary classification, the ratio of the relevant survival to contamination terms plays the role of the effective signal-to-noise ratio and shows up as a key quantity in our error analysis (Equation~\eqref{eqn:correctclassificationscenario} from Section~\ref{sec:proof_sketch}). 
When this ratio grows asymptotically to $\infty$, multiclass classification generalizes well. To the best of our knowledge, this is the first work that quantifies this effect of fewer informative samples per class and in what sense that makes multiclass classification harder than binary classification. The closest related work (\citep{multi_class_theory:Wang21}) only considers multiclass classification in the fixed finite class setting and consequently, doesn't compute exact dependencies on the number of classes $k$. 
We provide a more detailed comparison of our work with \citet{multi_class_theory:Wang21} and \citet{binary:Muth20} in \app{comparisonToWang}.


\section{Related Work}

The present work is situated within a larger stream of theoretical research trying to understand why overparameterized learning works and its limits. The limited page budget here forces brevity, but we recommend the recent surveys \citet{survey:bartlett2021deep, survey:belkin2021fit, survey:dar2021farewell} for further context.

Classically, by either operating in the underparameterized regime or by performing explicit regularization, we can force the training procedure to average out the harmful effects of training noise and thereby hope to obtain good generalization. 
The present cycle of seeking a deeper understanding began after it was observed that modern deep networks were overparameterized, capable of memorizing noise, and yet still generalized well, even when they were trained without explicit regularization \citep{reg:neyshabur2014search, reg:zhang2016understanding}. 
Experiments in \citet{reg:geiger2019jamming, reg:belkin2019reconciling} observed a double-descent behavior of the generalization error where in addition to the traditional U-shaped curve in the underparameterized regime, the error decreases in the overparameterized regime as we increase the number of model parameters. 
This double descent phenomenon is not unique to deep learning models and was replicated for kernel learning \citep{reg:belkin2018understand}. Further, the good generalization performance in the overparameterized regime cannot be explained by traditional worst-case generalization bounds based on Rademacher complexity or VC-dimension since the models have the capacity to fit purely random labels. 
Overparameterized models must therefore have some fortuitous combination of the model architecture with the training algorithm that leads us to a particular solution that generalizes well.  

To understand the phenomenon better, several works study the simpler setting of overparameterized linear regression.
The minimum-$\ell_2$ norm\footnote{The  minimum-$\ell_1$ norm interpolator has also been studied in \citet{reg:muthukumar2020harmless, reg:mitra2019understanding, reg:li2021minimum, reg:wang2021tight} and while sparsity-seeking behavior helps preserve the true signal (if the true pattern indeed depends only on a few features), it poses a challenge for the harmless absorption of noise since the desired averaging behaviour is not achieved fully \citep{reg:muthukumar2020harmless}.} interpolator is of particular interest since gradient descent on the squared loss has an implicit\footnote{In fact, there is an important complementary literature that brings out the implicit regularization performed by training methods, especially variants of gradient descent and stochastic gradient descent, and how the underlying architecture of the model shapes this implicit regularization \citep{implicit_bias:gunasekar2018characterizing, implicit_bias:srebro,implicit_bias:ji2019,  implicit_bias:woodworth2019kernel, implicit_bias:nacson2019convergence, implicit_bias:azizan2020study, implicit_bias:wu2020direction}.} bias towards this solution in the overparameterized regime \citep{implicit_bias:engl1996regularization} and has been studied extensively. (An incomplete list is \citet{reg:hastie2019surprises, reg:mei2019generalization,reg:bartlett2020benign, reg:belkin2020two, reg:muthukumar2020harmless, reg:bibas, reg:kobak2020optimal, reg:wu2020optimal, reg:richards2021asymptotics}.)
To generalize well, the underlying feature family must satisfy a balance between having a few important directions that sufficiently
favor the true pattern, and a large number of unimportant directions that can
absorb the noise in a harmless manner.

\subsection{High dimensional binary classification}

Both concurrently with and subsequent to the wave of analyses on overparameterized regression, researchers turned their attention to binary classification. 
A line of work poses the overparameterized binary classification problem as an optimization problem and analyzes it directly to obtain precise asymptotic behaviours of the generalization error  \citep{classification_cgmt:deng21, classification_cgmt:salehi2019impact, classification_cgmt:kammoun2021precise, classification_cgmt:taheri2020sharp, classification_cgmt:montanari2019generalization, classification_cgmt:kini2020analytic, classification_cgmt:taheri2021fundamental}.
 The key technical tool employed in these works is the Convex Gaussian Min-max Theorem and the resultant error formulas involve solutions to a system of non-linear equations that typically do not admit closed-form expressions. The generalization error of the max-margin SVM has also been analyzed directly by studying the iterates of gradient descent in  \citep{binary_classification:chatterji2021finite} and leveraging the implicit regularization perspective
of optimization algorithms.

However, although the above works did significantly enhance our understanding of binary classification in the overparameterized regime, a fundamental question was not answered:  ``Is classification easier than regression?" 
While the classification task is easier than the regression task at test time (regression requires us to correctly predict a real value while binary classification requires us to only predict its sign correctly), the training data for classification is less informative than that for regression since the labels are also binary. 
As described earlier, this question was answered in  \citet{binary:Muth20}, by exhibiting an asymptotic regime where binary classification error goes to zero, but the regression error does not. This was shown using Gaussian features with a bi-level covariance model. It turns out that the level of anisotropy (favoring of true features) required to perform regression correctly is significantly higher than that required for binary classification.

The key to the result in \citet{binary:Muth20} was the signal-processing inspired survival/contamination framework introduced in \citet{reg:muthukumar2020harmless} as a reconceptualization of the ``effective ranks'' perspective of \citet{reg:bartlett2020benign}. 
The survival concept relates to the shrinkage induced by the regularizing effect of having lots of features in the context of min-norm interpolation --- survival captures what is left of the true pattern after shrinkage. Contamination reflects the consequence of overparameterization when training via optimization: in addition to the true pattern, there is an infinite family of other\footnote{This is related to what is called the challenge of ``underspecification'' in ML \citep{d2020underspecification}, and this in turn is one aspect of the challenge of covariate shifts \citep{covariate_shift}.} false patterns (aliases) that also happen to explain the limited training data, and the optimizer ends up hedging its bet across the true pattern and these other competing false explanations. 
The learned false patterns contaminate the predictions on test points, and this can be quantified by the relevant standard deviation. 
For binary classification to succeed, what matters is that the survival exceed the contamination so that the sign of the prediction remains correct. 
Meanwhile, regression is harder since for regression to succeed, the survival must also tend to $1$. 




\subsection{Multiclass classification and the role of training loss function}

There is a large classical body of work on multiclass classification algorithms  \citep{algo_multiclass:Weston98, algo_multiclass:Bred99, algo_multiclass:Dietterich95, algo_multiclass:Crammer01, algo_multiclass:Lee04}, with further works giving computationally efficient algorithms for extreme multiclass problems with a huge
number of classes \citep{extreme:choromanska2013extreme, extreme:yen2016pd, extreme:rawat2019sampled}.  
Numerous theoretical works investigate the consistency of classifiers \citep{consistency:Zhang2004StatisticalAO, consistency:Pires16, consistency:Pires2013CostsensitiveMC, consistency:Tewari05, consistency:Chen06}. 
Finite-sample analysis of the generalization error in multiclass classification problems in the underparameterized regime has been studied in \citet{multi_class_theory:Kolt02, multi_class_theory:Gue02, algo_comparison:Allwein00, multi_class_theory:Li18, multi_class_theory:Cortes16, multi_class_theory:Lei15, multi_class_theory:Mau16, multi_class_theory:Lei19, multi_class_theory:Kuz14, multi_class_theory:Kuznetsov2015RademacherCM} and includes both data dependent bounds using Rademacher complexity, Gaussian complexity and covering numbers as well as data-independent bounds using the VC dimension. Recent work \citep{multi_class_theory:Thram20} leverages the Convex Gaussian Min-max Theorem to precisely characterize the asymptotic behaviour of the least-squares classifier in underparameterized multiclass classification.  

So, how different is multiclass classification from binary classification? 
The test time task is more difficult and for the same total number of training points, we have fewer positive training examples from each class. 
Several empirical studies comparing the performances of multiclass classification via learning multiple binary classifiers have been undertaken \citep{algo_comparison:Rifkin04, algo_comparison:Johannes02, algo_comparison:Allwein00}. The effects of the loss function while using deep nets to perform classification has also been investigated \citep{nets_loss_function:Hou16, nets_loss_function:Kry17, nets_loss_function:Kum18, nets_loss_function:Bos20, nets_loss_function:Dem20, nets_loss_function:Kline05, nets_loss_function:Hui20, label_imbalance}. Empirical evidence of least-squares minimization yielding competitive test classification performance to cross-entropy minimization has been presented in \citet{algo_comparison:Rifkin04, nets_loss_function:Hui20, nets_loss_function:Bos20}.

More recently, \citet{multi_class_theory:Wang21} makes progress towards bridging the gap between empirical observations and theoretical understanding by proving that in certain overparameterized regimes the solution to a multiclass SVM problem is identical to the one obtained by minimum-norm interpolation of one-hot encoded labels (equivalently, that gradient descent on squared loss leads to the same solution as gradient descent on cross-entropy loss as a result of implicit bias of these algorithms \citep{implicit_bias:engl1996regularization, implicit_bias:ji2019, implicit_bias:srebro}). In addition, \citet{multi_class_theory:Wang21} extends the analysis presented in \citet{binary:Muth20} for the binary classification problem to the multiclass problem with finitely many classes via an interesting reduction to analyzing a finite set of pairwise competitions, all of which must be won for multiclass classification to succeed. (We give further comments on the relationship of the present paper with \citet{multi_class_theory:Wang21} in \app{comparisonToWang}.)





\section{Problem setup}
\label{sec:setup}
We consider the multiclass classification problem with $k$ classes. The training data consists of $n$ pairs $\{\vec{x}_i,\ell_i\}_{i=1}^n$ where $x_i \in \mathbb{R}^d$ are i.i.d Gaussian vectors drawn from distribution,
\begin{align}
    \vec{x}_i \sim \mathcal{N}(0, I_d). \label{eq:unweightedxs}
\end{align}
We make the following assumption on how the labels $\ell_i \in [k]$ are generated. 

\begin{assumption} \textbf{1-sparse noiseless model}\footnote{A more generic model is $\ell_i = \argmax_{m \in [k]} \boldsymbol{\mu}_m \tran \vec{x}_i$ but we consider the simplified case where $\boldsymbol{\mu}_m$ are unit vectors, and are orthogonal to each other for different classes $m$. For the bi-level ensemble model (Definition~\ref{def:bilevel}) that we consider in this paper, this is exactly equivalent to the 1-sparse model defined here if we make the further assumption that the $\boldsymbol{\mu}_m$ have no support outside of favored features.}\\
\label{assumption:1sparse}
 The class labels $\ell_i$ are generated based on which of the first $k$ dimensions of a point $\vec{x}_i$ has the largest value,
\begin{align}
    \ell_i = \argmax_{m \in [k]} \vec{x}_i[m].
    \label{eq:truelabels}
\end{align}
\end{assumption}
We use the notation $x_i[m]$ to refer to the $m^{th}$ element of vector $\vec{x}_i$. 
For clarity of exposition, we make explicit a feature weighting that transforms the training points as follows:
\begin{align}
    x^w_i[j] =  \sqrt{\lambda_j} x_i[j] \quad \forall j \in [d]. \label{eq:lambdas}
\end{align}
Here $\lambdabold \in \mathbb{R}^d$ contains the squared feature weights. 
The feature weighting serves the role of favoring the true pattern, something that is essential for good generalization.\footnote{Our weighted feature model is equivalent to the one used in other works (e.g. \citep{binary:Muth20}) that assume that the covariates come from an anisotropic Gaussian with a covariance matrix that favors the truly important directions.} 

The weighted feature matrix $\mathbf{X}^w \in \mathbb{R}^{n \times d}$ is given by,
\begin{align}
    \mathbf{X}^w &= \begin{bmatrix} \vec{x}^w_1 & \dots &\vec{x}^w_j & \dots &  \vec{x}^w_n \end{bmatrix}\tran = \mat{ \sqrt{\lambda_1} \vec{z}_1 & \dots & \sqrt{\lambda_j} \vec{z}_j &  \dots & \sqrt{\lambda_d} \vec{z}_d } \label{eq:weightedxmatrix},
\end{align}
where $\vec{z}_j \in \mathbb{R}^n$ contains the $j^{th}$ features from the $n$ training points. Note that $\vec{z}_j \sim \mathcal{N}(0, I_n)$ are i.i.d Gaussians. 
We use a one-hot encoding for representing the labels as the matrix $\vec{Y}^{oh} \in \mathbb{R}^{n \times k}$,
\begin{align}
    \vec{Y}^{oh} &= \begin{bmatrix} \vec{y}^{oh}_1 & \dots &\vec{y}^{oh}_m & \dots & \vec{y}^{oh}_k \end{bmatrix} \label{eq:yohmatrix},
\end{align}
where,\begin{align}
    y^{oh}_m[i] &= \begin{cases} 1, &  \text{if} \ \ell_i = m \\ 0, & \text{otherwise} \end{cases} \label{eq:yoh}. 
\end{align}
A zero-mean variant of the encoding where we subtract the mean $\frac{1}{k}$ from each entry is denoted:
\begin{align}
    \vec{y}_m = \vec{y}_m^{oh} - \frac{1}{k} \mathbf{1} \label{eq:zeromeany}.
\end{align}

Our classifier consists of $k$ coefficient vectors $\hat{\vec{f}}_m$ for $m \in [k]$ that are learned by minimum-norm interpolation of the zero-mean one-hot variants using the weighted features.\footnote{The classifier learned via this method is equivalent to those obtained by other natural training methods under sufficient overparameterization \citep{multi_class_theory:Wang21}.}
\begin{align}
\hat{\vec{f}}_m &= \arg \min_\vec{f} \| \vec{f} \|_2\\
\text{s.t.}\ & \ \mathbf{X}^w \vec{f} = \vec{y}^{oh}_m - \frac{1}{k}\mathbf{1}. \label{eq:interpolateadjustedonehot}
\end{align}
We can express these coefficients in closed form as,
\begin{align} \label{eqn:featurecoefficients}
        \hat{\vec{f}}_m =  (\mathbf{X}^w)\tran \left(\mathbf{X}^w (\mathbf{X}^w)\tran \right)^{-1} \vec{y}_m.
\end{align}

On a test point $\xtest \sim \mathcal{N}(0,I_d)$ we predict a label as follows: First, we transform the test point into the weighted feature space to obtain $\xwtest$ where $x^w_{test}[j] = \sqrt{\lambda_j} x_{test}[j]$ for $j \in [d]$. Then we compute $k$ scalar ``scores'' and assign the class based on the largest score as follows:
\begin{align}
    \hat{\ell} = \argmax_{1 \leq m \leq k } \hat{\vec{f}}_m\tran \xwtest.
\end{align}
The true label of the test point is $\ell_{test} = \argmax_{1 \leq m \leq k }  x_{test}[m].$ A misclassification event $\mathcal{E}_{err}$ occurs iff
\eqn{
    \argmax_{1 \le m \le k} x_{test}[m] \ne \argmax_{1 \le m \le k} \hat{\vec{f}}_m\tran \xwtest.
}
In our work we determine sufficient conditions under which the probability of misclassification (computed over the randomness in both the training data and test point) goes to zero in an asymptotic regime where the number of training points, number of features, number of classes and feature weights scale according to the bi-level ensemble model. 
\begin{definition} (\textbf{Bi-level ensemble}):
\label{def:bilevel}
The bi-level ensemble is parameterized by $p,q,r$ and $t$ where $p > 1$, $0 \leq r < 1$, $0 < q < (p - r)$ and $0 \leq  t < r$. Here, parameter $p$ controls the extent of overparameterization, $r$ determines the number of favored features, $q$ controls the weights on favored features 
and $t$ controls the number of classes. The number of features ($d$), number of favored features ($s$), number of classes ($k$) and feature weights ($\sqrt{\lambda_j}$) all scale with the number of training points ($n$) as follows:
\begin{align}
    d = \lfloor n^p \rfloor, s = \lfloor n^r \rfloor, a = n^{-q}, k = c_k \lfloor n^{t} \rfloor, \label{eq:bilevelparamscaling}
\end{align}
where $c_k$ is a positive integer.
The feature weights are given by,
\begin{align}
\sqrt{\lambda_j} = \begin{cases} \sqrt{\frac{ad}{s}}, & 1 \leq j \leq s \\
                        \sqrt{\frac{(1-a)d}{d-s}}, & \mathrm{otherwise}\end{cases} . \label{eq:bilevellambdascaling}
\end{align}
\end{definition}
We provide a visualization of the bi-level model in Figure~\ref{fig:bilevelmodel}.

\begin{figure*}[ht!]
  \centering
  \includegraphics[width=0.6\columnwidth]{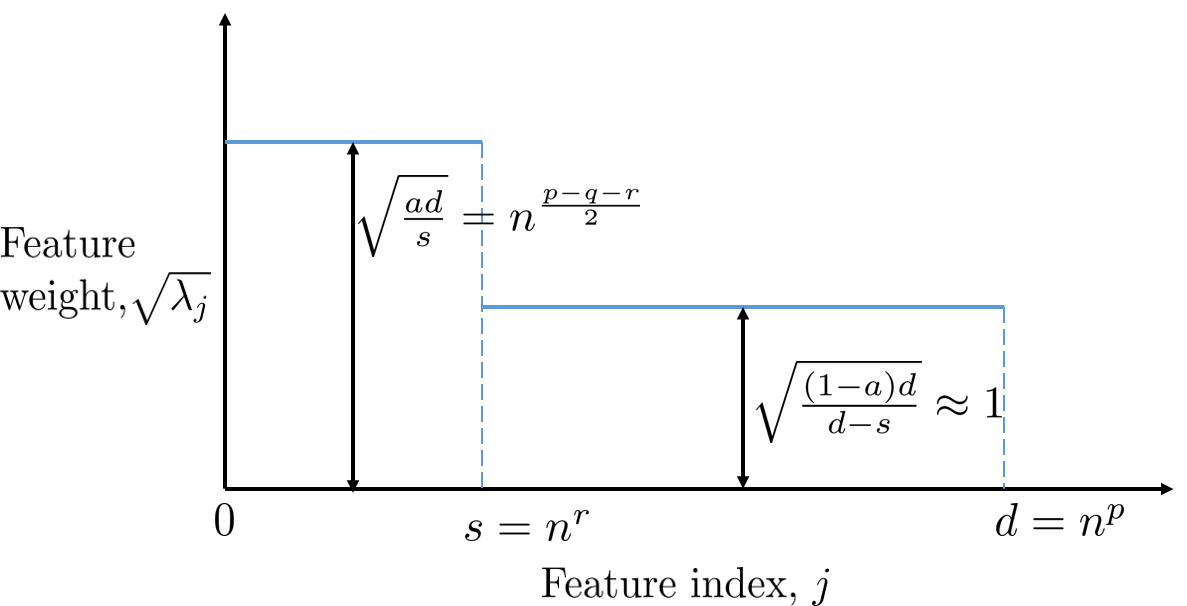}
  \caption{Bi-level feature weighting model. The first $s$ features have a higher weight and are favored during minimum-norm interpolation. These can be thought of as the square-roots of the eigenvalues of the feature covariance matrix in a Gaussian model for the covariates as in \citet{reg:bartlett2020benign}. 
  }
  \label{fig:bilevelmodel}
\end{figure*}

\section{Main result}
\begin{restatable}{theorem}{regimes} (\textbf{Asymptotic classification region in the bi-level model}):
\label{theorem:regimes}
Under the bi-level ensemble model~\ref{def:bilevel}, when the true data generating process is 1-sparse (Assumption~\ref{assumption:1sparse}), the probability of misclassification  $P(\mathcal{E}_{err}) \rightarrow 0$ as $n\rightarrow \infty$ if the following conditions hold:
\begin{align}
    t &< \min\left( r, 1-r, p+1-2(q+r), p-2, 2q+r-2 \right) \\
    q + r &> 1.
\end{align}

\end{restatable}


\newpage

\begin{figure*}[h!]
  \centering
  \includegraphics[width=1\columnwidth]{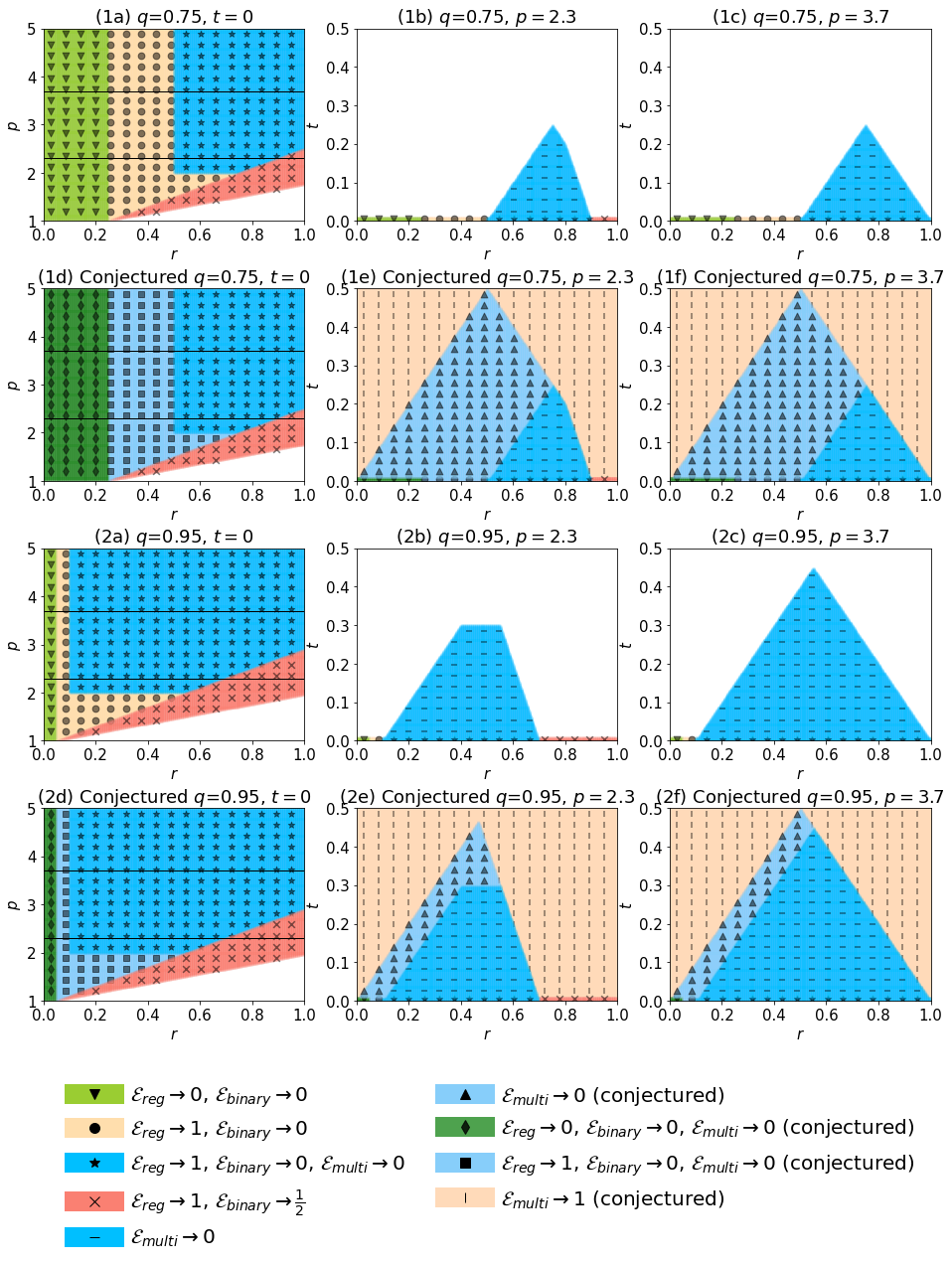}
  \caption{Visualization of the bi-level regimes in four dimensions $p,q,r,t$. (1a) and (2a) contrast multiclass classification with finite classes to binary classification and regression. The horizontal lines $p=2.3$ and $p=3.7$ correspond to the slices visualized in (1b), (1c), (2b) and (2c). The conjectured regimes are visualized in  (1d), (1e), (1f), (2d), (2e) and (2f).}
  \label{fig:bilevel_regimes}
\end{figure*}

\newpage 

Note that from \citet{binary:Muth20}, the condition $q+r>1$ corresponds to the regime where the corresponding regression problem\footnote{The corresponding regression problem is one where the true real number to be predicted is defined by a linear combination of favored features.} does not generalize well and thus our result shows that multiclass classification can generalize in regimes where the regression problem does not. 
Figure~\ref{fig:bilevel_regimes} visualizes the regimes by considering slices of the four dimensional scaling parameter space of $p,q,r$ and $t$. 
(1a) and (2a) fix the value of $q$ to $0.75$ and $0.95$ respectively and contrast the multiclass problem with a fixed finite number of classes ($t=0$) to the binary classification and regression problems. 
From these plots we observe that if we fix $p,q,t$ and increase $r$, i.e.~increasing how many features are favored (and thereby favoring each of them less), we transition from the regime where both regression and binary classification work, into the regime where binary classification works but regression does not, then the regime where this paper can prove multiclass classification works and finally to the regime where neither regression nor binary classification works. 

In Figure~\ref{fig:bilevel_regimes}, subplots (1b),(1c),(2b) and (2c) each visualize a slice along the $r$ and $t$ (class scaling) dimensions with fixed  $p$ and $q$. 
The x axis itself in these plots corresponds to a fixed finite classes setting. 
From (1b) we observe that the right-hand boundary of the region where multiclass classification generalizes well contains two slopes. 
These slopes arise from the two conditions $t < 1-r$ and $t < p+1 -2(q+r)$ in Theorem~\ref{theorem:regimes} and are a result of either contamination from favored (but not true) features dominating or contamination from the unfavored features dominating.
In (1c) we are in the regime where binary classification works for all values of $r<1$. However, as we increase $t$, eventually multiclass classification stops working.\footnote{To be precise, what the region actually illustrates is that our proof approach stops being able to show that multiclass classification works. In the Conclusion section, we conjecture where we believe that multiclass classification actually stops working. The conjectured regions are illustrated in (1e),(1f),(2e) and (2f).}

When we go from the binary problem to a multiclass problem with $k$ classes, the survival drops by a factor of $k$ as a consequence of having only $\frac{1}{k}$ fraction of positive training examples per class. 
This is because the one-hot labels we interpolate while training have fewer large values close to 1 that are able to positively correlate with the true feature vector. 
Having fewer positive exemplars also reduces the total energy in the training vector by a factor of $k$, and because of the square-root relationship of the standard deviation to the energy, the contamination only shrinks by a factor of $\sqrt{k}$. 
The overall survival/contamination ratio decreases by a factor of $\sqrt{k}$ making the multiclass classification task more difficult.\footnote{This is also responsible for contamination due to favored features being able to cause errors. 
For binary classification, because the true feature survival is constant (depending only on the level of label noise), the survival can always asymptotically overcome any contamination from other favored features \citep{binary:Muth20}.} 
An interesting observation here is the amount of favoring required for good generalization is linked to the number of positive training examples per class. Indeed, if we consider a setting where the binary classification problem generalizes well, and we switch to the $k$ class multiclass problem, then by increasing the number of training samples $k$ fold (and thus matching the number of positive training examples per class in the multiclass case to the binary case) and keeping the number of features and feature weights constant we can generalize well for multiclass classification. (\app{scaleexamples} elaborates on this phenomenon, as well as why it is somewhat surprising.) 

Next, we present a brief overview of our proof that utilizes the survival/contamination analysis framework from \citet{binary:Muth20} along with a typicality-inspired argument where the feature margin (difference between largest and second largest feature) on the test point plays a key role. The complete proof is provided in \apps{Appendices \ref{app:proof}, \ref{app:common_results}, \ref{app:utility_bounds}, and \ref{app:misclassification_events}}. 
\subsection{Proof sketch}
\label{sec:proof_sketch}
Assume without loss of generality that for the test point $\xtest \sim \mathcal{N}(0, I_d)$, the true class is $\alpha$ for some $\alpha \in [k]$. Let $\xwtest$ be the weighted version of this test point. 
A necessary and sufficient condition for classification error is that for some $\beta \neq \alpha, \beta \in [k]$,
\begin{align}
    \fhat_\a[\a] x^w_{test}[\a] + \fhat_\a[\b] x^w_{test}[\b] + \sum_{j \notin \{\a, \b \}} \fhat_\a[j] x^w_{test}[j] &< \fhat_\b[\a] x^w_{test}[\a] \nonumber \\
    & \quad + \fhat_\b[\b] x^w_{test}[\b] + \sum_{j \notin \{\a, \b \}} \fhat_\b[j] x^w_{test}[j].
\end{align}
By converting into the unweighted feature space we obtain
\begin{align}
    \lambda_\a \hhat_{\a,\b}[\a] x_{test}[\a] - \lambda_\b \hhat_{\b,\a}[\b] x_{test}[\b] &<  \sum_{j \notin \{\a, \b \}} \lambda_j \hhat_{\b,\a}[j] x_{test}[j],
\end{align}
where
\eqn{
    \hhat_{\a,\b}[j] = \lambda_j^{-1/2}(\hat{f}_\a[j] - \hat{f}_\b[j]).
}
Performing some algebraic manipulations and because $\lambda_\a = \lambda_\b = \lambda$ since both $\a$ and $\b$ are favored features, we can rewrite this as
\begin{align}
    \frac{\lambda \hhat_{\a,\b}[\a]}{ \cn_{\a,\b}}\left((x_{test}[\a] - x_{test}[\b]) + x_{test}[\b]\frac{\hhat_{\a,\b}[\a] - \hhat_{\b,\a}[\b]}{\hhat_{\a,\b}[\a]}\right) \nonumber \\
     < \frac{1}{ \cn_{\a,\b}} &\sum_{j \notin \{\a, \b \}} \lambda_j \hhat_{\b,\a}[j] x_{test}[j], 
\end{align}
where
\eqn{
\cn_{\a,\b} = \sqrt{\left(\sum_{j \notin \{\a, \b\}} \lambda_j^2 (\hhat_{\b,\a}[j])^2 \right)}. \label{eq:cn}
}
We divide by $\cn_{\a,\b}$ to normalize the RHS above to have a standard normal distribution.  
Next, by removing the dependency on $\beta$, we obtain a sufficient condition for correct classification:
\eqn{
    \label{eqn:correctclassificationscenario}
   \underbrace{\frac{\min_\beta \lambda \hhat_{\a,\b}[\a]}{\max_\b \cn_{\a,\b}}}_{\text{SU/CN ratio}} \left( \underbrace{\min_\b \left( x_{test}[\a] - x_{test}[\b]\right)}_{\text{closest feature margin}} - \underbrace{\max_\b |x_{test}[\b]|}_{\text{largest competing feature}} \cdot \underbrace{\max_\b \left|\frac{\hhat_{\a,\b}[\a] - \hhat_{\b,\a}[\b]}{\hhat_{\a,\b}[\a]} \right|}_{\text{survival variation}} \right) \nonumber \\
   > \underbrace{\max_{\beta} \frac{1}{\cn_{\a,\b}}\left(\sum_{j \notin \{\a, \b \}} \lambda_j \hhat_{\b,\a}[j] x_{test}[j]  \right)}_{\text{normalized contamination}}. 
}
Here the min and max are over all competing features: $1 \le \beta \le k, \beta \ne \alpha$ and the sum is over all $d$ feature indices except $\alpha$ and $\beta$, but we simplify the notation for convenience. 
We show via intermediate lemmas introduced in \app{proof} that under the conditions specified in Theorem~\ref{theorem:regimes}, with sufficiently high probability\footnote{This is where we leverage the idea of typicality-style proofs in information theory \citep{Cover2006} to avoid unnecessarily loose union bounds that end up being dominated by the atypical behavior of quantities. In our case, by pulling the feature margin out explicitly, we can just deal with its typical behavior. Similarly, the typical behavior of the largest competing feature and the true feature is all that matters.}, the relevant survival to contamination 
\emph{SU/CN ratio} grows at a polynomial rate $n^{v}$ for some $v>0$, the \emph{closest feature margin} shrinks at a less-than-polynomial rate $1/\sqrt{\ln nk}$, the \emph{survival variation} decays at a polynomial rate $n^{-u}$ for some $u>0$. Further, the magnitudes of the \emph{largest competing feature} and the \emph{normalized contamination} are no more than $\sqrt{\ln(nk)}$. 

This implies that the left-hand side of Equation~\eqref{eqn:correctclassificationscenario} grows at a polynomial rate  $n^v$ (ignoring logarithmic terms)  and dominates the right-hand side which grows at the much slower rate $\sqrt{\ln nk}$. 
A survival/contamination ratio also plays a key role in the analysis of the binary classification problem in \citet{binary:Muth20} but in the multiclass setting, we additionally have the survival variation term and feature margin playing important roles since we are comparing different scores while predicting the class label. For correct classification, the survival/contamination ratio must be sufficiently large, the survival variation must be small enough and the feature margin must be sufficiently large.


\section{Conclusion} \label{sec:conclusion}
In this work we compute sufficient conditions for good generalization of multiclass classification in a bi-level overparameterized linear model with Gaussian features. We observed that multiclass classification can generalize even when the regression problem does not generalize (for $q+r>1$). 
Further, the multiclass problem is ``harder'' than the binary problem because we have fewer positive training examples per class. 
The nature of the training data complicates our analysis in the multiclass setting since the true class labels are generated by comparing $k$ features and thus we no longer have independence of the encoded class label $y$ with any of these features. 
This becomes relevant when we compute bounds on the survival and contamination quantities since the Hanson-Wright inequality \citep{rudelson2013hanson} is no longer applicable directly on the quantities of interest as was the case for the binary classification problem in prior work \citep{binary:Muth20}. 
As a consequence of working around this non-independence we believe that our sufficient conditions for good generalization in the regime $q+r>1$ are loose. 

Even though in our work we focus on the regime where regression does not work, $q+r>1$, we can extend the analysis to the regime where $q+r<1$ by grinding through the expressions for survival and contamination in this regime. Even in this regime, for multiclass training data, survival is of the order $\frac{1}{k}$ while contamination scales similarly to the regime $q+r>1$. Thus, while it is true that for binary classification or a fixed number of classes, the regime where regression works is a regime where classification also works, this need not be true if there are too many classes.

We conjecture that the following is a set of necessary and sufficient conditions for asymptotically good generalization (We elaborate on this in \app{conjectured_loose}):
\begin{conjecture} (\textbf{Conjectured bi-level regions}):
\label{conjecture:regimes}
Under the bi-level ensemble model~\ref{def:bilevel}, when the true data generating process is 1-sparse (Assumption~\ref{assumption:1sparse}), as $n\rightarrow \infty$, the probability of misclassification event $P(\mathcal{E}_{err})$ behaves as follows:
\begin{align}
     P(\mathcal{E}_{err})& \rightarrow \begin{cases} 0,& \ \mathrm{if} \  t < \min\left(r, 1-r, p+1-2(q+r)\right) \\
      1,& \ \mathrm{if} \ t > \min\left(r, 1-r, p+1-2(q+r)\right) \end{cases}.
      \label{eq:conjectured_regimes}
\end{align}
\end{conjecture}

 The conjectured regions are visualized in (1d),(1e),(1f),(2d),(2e) and (2f) in Figure~\ref{fig:bilevel_regimes}. 
 Subfigures (1d) and (2d) illustrate that we believe multiclass classification with finitely many classes works if binary classification works. Further, comparing (1e) to (2e) when we increase $q$, the conjectured parameter region where multiclass classification works shrinks since we decrease the amount of favoring of true features. 
 Interestingly, the nature of the looseness in our approach is such that our proof technique is able to recover a larger fraction of the conjectured region for larger $q$ which intuitively is a result of less favoring leading to stronger concentration of certain random quantities. Tightening the potential looseness in our analysis and proving the converse result by computing sufficient conditions for poor generalization of multiclass classification are interesting avenues of future work.



Recent work from \citet{multi_class_theory:Wang21} provides an analysis of the generalization error of the minimum-norm interpolation of one-hot labels for multiclass classification with Gaussian features. 
Using the bi-level model, the authors present parameter regimes where multiclass classification error goes to zero asymptotically, considering only the fixed finite classes setting ($t = 0$ in our model). 
They also show that the minimum-norm interpolating solution \eqref{eq:interpolateadjustedonehot} is typically identical to the solution obtained via one-vs-all SVM and multi-class SVM (and thus gradient descent on cross-entropy loss due to its implicit bias \citep{implicit_bias:ji2019, implicit_bias:srebro}) under sufficient overparameterization, even in the case when $t > 0$, so the number of classes $k = n^t$ grows with $n$, as long as it does not grow too rapidly. Under our bi-level model (Definition~\ref{def:bilevel}) the relevant condition from \citet{multi_class_theory:Wang21} for when the SVM solution matches minimum-norm interpolation (MNI) can be expressed as:
\begin{align}
    0 < t < \frac{ q + r -1}{2} \label{eqn:mnisvmequalregime}.
\end{align}
More details on the setup of \citet{multi_class_theory:Wang21}, as well as a derivation of \eqref{eqn:mnisvmequalregime} are in \app{comparisonToWang}.

\begin{figure*}[h!]
  \centering
  \includegraphics[width=1\columnwidth]{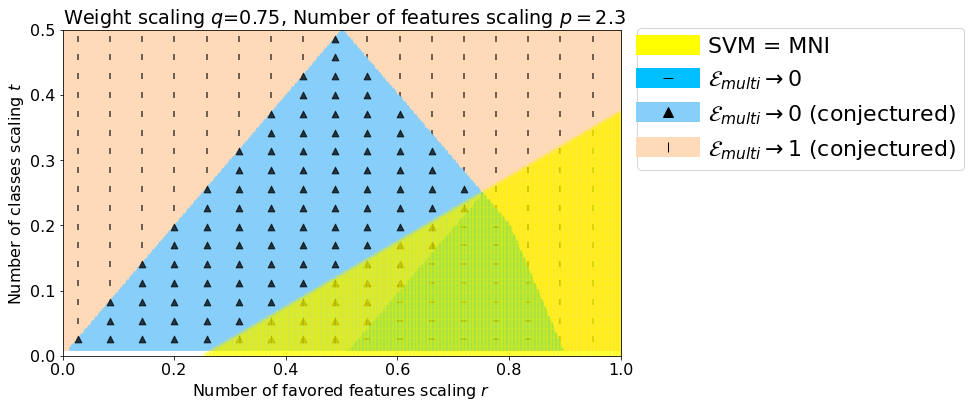}
  \caption{Visualization of regime where SVM solution is identical to MNI solution.}
  \label{fig:svm_equivalence}
\end{figure*}

In Figure~\ref{fig:svm_equivalence} we plot our provable as well as conjectured regimes alongside the regime from \eqref{eqn:mnisvmequalregime}. Notice the overlap. Thus our result is not limited only to the minimum-norm interpolator, but in fact holds for other training methods when the problem is sufficiently overparameterized. 
In this sense, the results in \citet{multi_class_theory:Wang21} and the present paper can be read together to tell a more full story of overparameterized multiclass classification. 
The behavior of the SVM solution in the conjectured region where $\varepsilon_{multi} \to 1$, but where it is not known whether $\text{SVM} = \text{MNI}$, is left for future work.

Further, although the present analysis focuses on solutions that exactly interpolate the training data, we can extend our results to account for additional ridge regularization by viewing ridge regularization as minimum-norm interpolation using augmented contamination-free features as in the Appendix of \citet{reg:muthukumar2020harmless} and computing bounds leveraging tools from \citet{reg:tsigler2020benign}. 
Our assumption of the strict bi-level weighting model is largely to simplify the calculations and by substituting terms appropriately in our lemmas from \app{proof}, it should be possible to compute results for other weighting models. Finally, exploring the new phenomena that can be encountered as we go beyond the 1-sparse noiseless model is an exciting direction for future work.

\begin{ack}
We are grateful to our earlier collaborators Vidya Muthukumar, Misha Belkin, Daniel Hsu, and Adhyyan Narang. In addition, we want to thank the students and course staff for the Fall 2020 iteration of Berkeley's CS189 machine learning course, where we had adapted ideas from \citet{reg:muthukumar2020harmless, binary:Muth20} in teaching the foundations of modern machine learning --- the need for the present paper became more clear during this process.

We gratefully acknowledge the support from ML4Wireless center member companies and NSF grants AST-2132700 and AST-2037852 for making this research possible.
\end{ack}

\bibliographystyle{plainnat} 
\bibliography{references}

\newpage
\appendix


The appendix is organized as follows. 
In Appendix~\ref{app:notation} we provide a table of notations used throughout the paper.
Appendix~\ref{app:proof} provides an overall proof for Theorem~\ref{theorem:regimes} by introducing some intermediate lemmas and assuming they hold. 
Appendix~\ref{app:common_results} introduces some key tools that we need and Appendix~\ref{app:utility_bounds} leverages those tools to build towards a proof of these intermediate lemmas by introducing some helper results that are needed to deal with the key challenge posed by multiclass training data. Appendix~\ref{app:misclassification_events} actually proves the intermediate lemmas used in Appendix~\ref{app:proof} and completes the proof. Appendix~\ref{app:conjectured_loose} discusses the potential looseness in our analysis and describes how we obtained Conjecture~\ref{conjecture:regimes}. Appendix~\ref{app:scaleexamples} elaborates on the effect of fewer number of positive training examples per class in the multiclass setting and investigates an alternative setting where the total number of positive training examples per class is kept constant while we increase the number of classes. Finally, in Appendix~\ref{app:comparisonToWang} we provide a more detailed comparison of our work with \citet{multi_class_theory:Wang21} and \citet{binary:Muth20}.

Throughout the appendix we will assume that $n$ is large enough for asymptotic behavior to kick in.  
We also will introduce various universal positive constants, indexed as $c_i$. These constants are all independent of $n$, and constants with the same index are to be treated as equal throughout this Appendix.


\section{Notation}\label{app:notation}
We summarize the notation used in the problem setup (as well as some terms defined later) as follows:
\begin{table}[h!]
  \caption{Notation}
  \label{tab:notation}
  \centering
  \begin{tabular}{llll}
    \toprule
    Symbol & Definition & Dimension & Source \\
    \midrule
    $\nc$ & Number of classes & Scalar & Sec. \ref{sec:setup} \\
    $n$ & Number of training points & Scalar & Sec. \ref{sec:setup} \\
    $d$ & Dimension of each point --- the total number of features & Scalar & Sec. \ref{sec:setup} \\
    $s$ & The number of favored features & Scalar & Def. \ref{def:bilevel} \\
    $p$ & Parameter controlling overparameterization ($d = n^p$) & Scalar & Def. \ref{def:bilevel} \\
    $r$ & Parameter controlling the number of favored features ($s = n^r$) & Scalar & Def. \ref{def:bilevel} \\
    $a$ & Parameter controlling the favored weights ($a = n^{-q}$) & Scalar & Def. \ref{def:bilevel} \\
    $t$ & Parameter controlling the number of classes ($\nc = c_k n^t$) & Scalar & Def. \ref{def:bilevel} \\
    $c_k$ & The number of classes when $t=0$ ($\nc = c_k n^t$) & Scalar & Def. \ref{def:bilevel} \\
    $\lambda_j$ & Squared weight of the $j$th feature & Scalar & Def. \ref{def:bilevel} \\
    $\vec{x}_i$ & $i$th training point (unweighted) & Length-$n$ vector & Eqn. \ref{eq:unweightedxs} \\
    $\ell_i$ & Class label of $i$th training point & Scalar & Eqn. \ref{eq:truelabels} \\
    $\vec{x}_i^w$ & $i$th training point (weighted) & Length-$n$ vector & Eqn. \ref{eq:lambdas} \\
    $\Xw$ & Weighted feature matrix & $(n \times d)$-matrix & Eqn. \ref{eq:weightedxmatrix} \\
    $\zj$ & The collected $j$th features of all training points & Length-$n$ vector & Eqn. \ref{eq:weightedxmatrix} \\
    $\vec{y}_m^{oh}$ & One-hot encoding of all the training points for label $m$ & Length-$n$ vector & Eqn.
    \ref{eq:yoh} \\
    $\mathbf{Y}^{oh}$ & One-hot label matrix & ($n \times \nc$)-matrix & Eqn. \ref{eq:yohmatrix} \\
    $\vec{y}_m$ & Zero-mean encoding of the training points for label $m$ & Length-$n$ vector & Eqn.     \ref{eq:zeromeany} \\
    $\hat{\vec{f}}_m$ & Learned coefficients for label $m$ using min-norm interpolation &  Length-$d$ vector & Eqn. \ref{eqn:featurecoefficients} \\
    $\vec{x}_{test}$ & A single test point & Length-$d$ vector & Sec. \hyperref[eqn:featurecoefficients]{4} \\
     $\vec{x}_{test}^w$ & A single weighted test point & Length-$d$ vector & Sec. \hyperref[eqn:featurecoefficients]{4} \\
    $\Abold$ & $\Abold = \Xw(\Xw)\tran$ & ($n \times n$)-matrix & Eqn. \ref{eqn:Adefn} \\
    $\mu_i(\Abold)$ & The $i$th eigenvalue of matrix $\Abold$, sorted in descending order & Scalar & App.  \hyperref[eq:learnedcoeffs]{B}\\
    $\boldsymbol{\Lambda}$ & Matrix of squared feature weights: $\mathrm{diag}(\lambda_1, \lambda_2, \dots, \lambda_d)$ & ($d \times d$)-matrix & App. \hyperref[eq:learnedcoeffs]{B} \\
    $\hhat_{\a,\b}$ & Relative survival $\hhat_{\a,\b}[j] = \lambda_j^{-1/2}(\hat{f}_\a[j] - \hat{f}_\b[j])$ & Length-$d$ vector & Eqn. \ref{eq:hhatdefna} \\
    $\cn_{\a,\b}$ & 
    Normalizing factor $\cn_{\a,\b}= \sqrt{\left(\sum_{j \notin \{\a, \b\}} \lambda_j^2 (\hhat_{\b,\a}[j])^2\right)}$ & Scalar & Eqn. \ref{eq:cndefn} \\
    $\norm{\cdot}_{\psi_2}$ & 
    The sub-Gaussian norm of a scalar random variable & Scalar & Eqn. \ref{eq:subgaussiannorm} \\
       $\mubar$ & 
    Center of the eigenvalue bounds for $\Ainvbold$, $\mubar = \frac{1}{\sum_j \lambda_j} $& Scalar & Eqn. \ref{eq:mubar} \\
           $\Diamond$ & 
    Deviation term in eigenvalue bounds for $\Abold$ & Scalar & Eqn. \ref{eq:diamond2} \\       $\delmu$ & 
     Deviation term in eigenvalue bounds for $\Ainvbold$ & Scalar & Eqn. \ref{eq:delmu} \\
    \bottomrule
  \end{tabular}
\end{table}



\section{Proof of Theorem~\ref{theorem:regimes}}
\label{app:proof}
We restate Theorem~\ref{theorem:regimes}, our main result, here for convenience:
\regimes*


Our proof that utilizes the survival/contamination analysis framework from \citet{binary:Muth20} along with a typicality-inspired argument where the feature margin (difference between largest and second largest feature) on the test point plays a key role. 

Assume without loss of generality that for the test point $\xtest \sim \mathcal{N}(0, I_d)$, the true class is $\alpha$ for some $\alpha \in [k]$. 
Let $\xwtest$ be the weighted version of this test point. 
A necessary and sufficient condition for classification error is that for some $\beta \neq \alpha, \beta \in [k]$, the score associated with class $\beta$ is higher than the score associated with class $\alpha$. Pulling out the key terms associated with the $\alpha$ and $\beta$ weighted features, we get:
\begin{align}
    \fhat_\a[\a] x^w_{test}[\a] + \fhat_\a[\b] x^w_{test}[\b] + \sum_{j \notin \{\a, \b \}} \fhat_\a[j] x^w_{test}[j] &< \fhat_\b[\a] x^w_{test}[\a] \nonumber \\
    & \quad + \fhat_\b[\b] x^w_{test}[\b] + \sum_{j \notin \{\a, \b \}} \fhat_\b[j] x^w_{test}[j]\\
     \implies (\fhat_\a[\a] - \fhat_\b[\a]) x^w_{test}[\a] - (\fhat_\b[\b] - \fhat_\a[\b]) x^w_{test}[\b] &<  \sum_{j \notin \{\a, \b \}} (\fhat_\b[j] - \fhat_\a[j]) x^w_{test}[j]. \\
\end{align}
Note that $\sum_{j \in {\a, \b}}$ refers to the sum over all feature indices $1$ to $d$ excluding $\a$ and $\b$.

By converting into the unweighted feature space we obtain,
\begin{align}
    \lambda_\a \hhat_{\a,\b}[\a] x_{test}[\a] - \lambda_\b \hhat_{\b,\a}[\b] x_{test}[\b] &<  \sum_{j \notin \{\a, \b \}} \lambda_j \hhat_{\b,\a}[j] x_{test}[j],
\end{align}
where we introduce the short-hand notation,
\eqn{
    \hhat_{\a,\b}[j] &= \lambda_j^{-1/2}(\hat{f}_\a[j] - \hat{f}_\b[j])  \label{eq:hhatdefna} \\
     \hhat_{\b,\a}[j] &= \lambda_j^{-1/2}(\hat{f}_\b[j] - \hat{f}_\a[j]) \label{eq:hhatdefnb}.
}

 Since both $\a$ and $\b$ are favored feature indices, by leveraging the definition of the bi-level model and denoting $\lambda_\a = \lambda_\b = \lambda$, we get
\begin{align}
   \lambda \left( \hhat_{\a,\b}[\a] x_{test}[\a] - \hhat_{\b,\a}[\b] x_{test}[\b] \right) &<  \sum_{j \notin \{\a, \b \}} \lambda_j \hhat_{\b,\a}[j] x_{test}[j].
 \end{align}
 
Next, we perform some algebraic manipulations,
\begin{align}
       \lambda \left( \hhat_{\a,\b}[\a] x_{test}[\a] - \hhat_{\b,\a}[\b] x_{test}[\b] \right) &<  \sum_{j \notin \{\a, \b \}} \lambda_j \hhat_{\b,\a}[j] x_{test}[j] \\
       \implies \lambda \hhat_{\a,\b}[\a](x_{test}[\a] - x_{test}[\b]) + \lambda x_{test}[\b](\hhat_{\a,\b}[\a] - \hhat_{\b,\a}[\b]) &< \sum_{j \notin \{\a, \b \}} \lambda_j \hhat_{\b,\a}[j] x_{test}[j] \\
    \implies \lambda \hhat_{\a,\b}[\a]\left((x_{test}[\a] - x_{test}[\b]) + x_{test}[\b]\frac{\hhat_{\a,\b}[\a] - \hhat_{\b,\a}[\b]}{\hhat_{\a,\b}[\a]}\right) &< \sum_{j \notin \{\a, \b \}} \lambda_j \hhat_{\b,\a}[j] x_{test}[j]. \label{eq:beforenormalize}
\end{align}


We divide both sides by the quantity $\cn_{\a,\b}$ defined as,
\eqn{
\cn_{\a,\b} = \sqrt{\left(\sum_{j \notin \{\a, \b\}} \lambda_j^2 (\hhat_{\b,\a}[j])^2 \right)}. \label{eq:cndefn}
}
This normalizes the RHS of \eqref{eq:beforenormalize} to have a standard normal distribution.
Thus, the necessary and sufficient condition for a misclassification error is for some $\beta \neq \alpha, \beta \in [k]$,
\begin{align}
    \frac{\lambda \hhat_{\a,\b}[\a]}{ \cn_{\a,\b}}\left((x_{test}[\a] - x_{test}[\b]) + x_{test}[\b]\frac{\hhat_{\a,\b}[\a] - \hhat_{\b,\a}[\b]}{\hhat_{\a,\b}[\a]}\right) &< \frac{1}{ \cn_{\a,\b}} \sum_{j \notin \{\a, \b \}} \lambda_j \hhat_{\b,\a}[j]x_{test}[j].
\end{align}
A sufficient condition for correct classification can then be obtained by ensuring that the smallest potential value of the LHS is still greater than the value of the RHS for all values of $\beta$. 
Thus, 
we obtain a sufficient condition for correct classification by appropriately minimizing or maximizing quantities over competing feature indices $ \beta \neq \alpha, \beta \in [k]$ (for notational convenience we simply denote this as $\min_\beta$ or $\max_\beta$). 
\eqn{
   \underbrace{\frac{\min_\beta \lambda \hhat_{\a,\b}[\a]}{\max_\b \cn_{\a,\b}}}_{\text{SU/CN ratio}} \left( \underbrace{\min_\b \left( x_{test}[\a] - x_{test}[\b]\right)}_{\text{closest feature margin}} - \underbrace{\max_\b |x_{test}[\b]|}_{\text{largest competing feature}} \cdot \underbrace{\max_\b \left|\frac{\hhat_{\a,\b}[\a] - \hhat_{\b,\a}[\b]}{\hhat_{\a,\b}[\a]} \right|}_{\text{survival variation}} \right) \nonumber \\
   > \underbrace{\max_{\beta} \frac{1}{\cn_{\a,\b}}\left(\sum_{j \notin \{\a, \b \}} \lambda_j \hhat_{\b,\a}[j] x_{test}[j]  \right)}_{\text{normalized contamination}}. 
}

We will show that under the conditions specified in Theorem~\ref{theorem:regimes}, with sufficiently high probability,  the relevant survival to contamination 
\emph{SU/CN ratio} grows at a polynomial rate $n^{v}$ for some $v>0$, the \emph{closest feature margin} shrinks at a less-than-polynomial rate $1/\sqrt{\ln nk}$, and the \emph{survival variation} decays at a polynomial rate $n^{-u}$ for some $u>0$. Further, the magnitudes of the \emph{largest competing feature} and the \emph{normalized contamination} are no more than $2\sqrt{\ln(nk)}$. Here, we leverage the idea of typicality-style proofs in information theory \citep{Cover2006} to avoid unnecessarily loose union bounds that end up being dominated by the atypical behavior of quantities. In our case, by pulling the feature margin out explicitly, we can just deal with its typical behavior. Similarly, the typical behavior of the largest competing feature and the true feature is all that matters.
Before we proceed with the rest of our proof we remind the reader of a few important definitions.  

Recall from \eqref{eqn:featurecoefficients} that our learned feature coefficients are
\begin{align}
        \hat{\vec{f}}_m =  (\mathbf{X}^w)\tran \left(\mathbf{X}^w (\mathbf{X}^w)\tran \right)^{-1} \vec{y}_m.
\end{align}

Let
\eqn{
    \Abold = \mathbf{X}^w (\mathbf{X}^w)\tran. \label{eqn:Adefn}
}
Then we can express our learned coefficients as
\eqn{
    && \hat{f}_m[j] = \sqrt{\lambda_j} \zj\tran\Ainv\ya,
    \label{eq:learnedcoeffs}
}
where $\zj \in \mathbb{R}^n$ contains the $j^{th}$ features of all $n$ training points. The rows of $\Xw$ are i.i.d. Gaussians with covariance matrix $\boldsymbol{\Lambda} = \mathrm{diag}(\lambda_1, \lambda_2, \dots, \lambda_d)$. Let $\mu_1(\Abold)$ denote the largest eigenvalue and $\mu_n(\Abold)$ denote the smallest eigenvalue of $\Abold$ respectively, with $\mu_i(\Abold)$ being the $i$-th largest eigenvalue of $\Abold$.

Next, we state a useful lemma adapted from \citet{reg:bartlett2020benign} that bounds the eigenvalues of $\Ainv$ . Subsequent lemmas will utilize these eigenvalue bounds. 
\begin{restatable}{lemma}{eigenvaluebounds}
\label{lemma:eigenvaluebounds} (\textbf{Eigenvalue bounds on $\Ainv$ adapted from \citet{reg:bartlett2020benign}}):\\ If $\boldsymbol{\Lambda}$ is such that $\Diamond \ll \sum_j \lambda_j$, then with probability at least $(1 - 2e^{-n})$,
\begin{align}
      \mubar - \delmu \leq \mu_n(\Ainv) \leq \mu_1(\Ainv) \leq \mubar + \delmu,
\end{align}
where,
\begin{align}
    \mubar &= \frac{1}{\sum_j \lambda_j} \label{eq:mubar}\\
    \Diamond &= \frac{32}{9} \left( \lambda_1 (1 + \ln 9) n   + \sqrt{(1 + \ln 9) n \sum_j \lambda_j^2} \right) \label{eq:diamond} \\
    \delmu &= \mubar \left( \frac{\Diamond}{\sum_j \lambda j} + \Theta\left( \frac{\Diamond}{\sum_j \lambda j} \right)^2 \right) \label{eq:delmu}.
\end{align}
Further this implies that with probability at least $(1 - 2e^{-n})$,
\begin{align}
    \abs{\mu_i(\Ainvbold - \mubar \In)}\leq \delmu 
\end{align}
for all $i \in [n]$.
\end{restatable}

The subsequent lemmas bound the feature margin, survival, contamination and survival variation terms, utilizing tools from \citep{reg:bartlett2020benign} and building on results from \citep{binary:Muth20}.

\begin{restatable}{lemma}{marginbound} (\textbf{Lower bound on the closest feature margin as $k \to \infty$}):
\label{lemma:marginbound}
For any constant $\eps > 0$, there exists a constant $\theta$ such that, for sufficiently large $k$ with probability at least $(1 - \eps)$,
\eqn{
    \min_{\b: 1 \le \b \ne \a \le k} \left( x_{test}[\a] - x_{test}[\b]\right) \ge \frac{\theta}{\sqrt{2\ln(k)}}.
}
Here, $\alpha$ is fixed and corresponds to the index of the true class --- i.e.~$\alpha$ corresponds to the index of the maximum feature among the first $k$ features.
\end{restatable}


\begin{restatable}{lemma}{constantmarginbound}  (\textbf{Lower bound on the closest feature margin when $k$ is constant}):
\label{lemma:constantmarginbound}
If $k = c_k$ for some fixed constant $c_k$, for any constant $\eps > 0$, there exists a constant $\eps' > 0$ such that
\eqn{
    \Pr\left(\min_{\beta, \gamma: 1 \le \beta \ne \gamma \le c_k} \abs{x_{test}[\beta] - x_{test}[\gamma]} \ge \eps' \right) \ge 1 - \eps.
}

Thus, with probability at least $(1 - \eps)$,
\eqn{
    \min_{\b: 1 \le \b \ne \a \le k} \left( x_{test}[\a] - x_{test}[\b]\right) \ge  \eps'.
}
Here, $\alpha$ is fixed and corresponds to the index of the true class --- i.e.~$\alpha$ corresponds to the index of the maximum feature among the first $k$ features.
\end{restatable}

\begin{restatable}{lemma}{featuredifferencebound}
\label{lemma:featuredifferencebound} (\textbf{Lower bound on relative survival of true feature}):
For any fixed $\beta \in [k]$, $\beta \neq \alpha$, with $\lambda_\a = \lambda_\b = \lambda$ we have with probability at least $\left(1 - 5/(nk)\right)$,
\eqn{
    \lambda \hhat_{\a,\b}[\alpha] \ge \lambda\left(\const{c60}\mubar\frac{n}{k}\sqrt{\ln(k)}
    - \const{c50}(\mubar \sqrt{n} \sqrt{\ln(nk)} + \delmu \cdot n/\sqrt{k})\right),}
for universal positive constants $\const{c50}$ and $\const{c60}$.
\end{restatable}

By substituting the asymptotic behavior of parameters from our bi-level ensemble model we get the following corollary:
\begin{restatable}{corollary}{asymptoticfeaturedifferencebound}
\label{corollary:asymptoticfeaturedifferencebound}
Under the bi-level ensemble model~\ref{def:bilevel}, for any fixed $\beta \in [k]$, $\beta \neq \alpha$, $\lambda_\a = \lambda_\b = \lambda$ if $t < 1/2$, $t < 2(q+r-1)$ and $1 < q + r < (p+1)/2$, with probability at least $\left(1 - 5/(nk)\right)$, 
\eqn{
    \lambda \hhat_{\a,\b}[\alpha] \ge \const{c120} n^{1-q-r-t} \sqrt{\ln(k)},
}
for universal positive constant $\const{c120}$.
\end{restatable}

\begin{restatable}{lemma}{contaminationbound} (\textbf{Upper bound on contamination}):
\label{lemma:contaminationbound}
For any fixed $\beta \in [k]$, $\beta \neq \alpha$, with probability at least $\left(1 - 7/(nk)\right)$,
\eqn{
    CN_{\a,\b} &\le \const{c30}(\mubar \sqrt{\frac{n}{\nc}}\cdot \sqrt{\ln(nd\nc)} + \delmu \cdot n/\sqrt{k}) \sqrt{\sum \lambda_j^2},
}
for universal positive constant \const{c30}.
\end{restatable}

As before, for our bi-level ensemble model we have the corollary:
\begin{restatable}{corollary}{asymptoticcontaminationbound}
\label{corollary:asymptoticcontaminationbound}
Under the bi-level model~\ref{def:bilevel}, in the regime $1 < q + r < (p+1)/2$,
with probability at least $\left(1 - 7/(nk)\right)$,  
\eqn{
    CN_{\a,\b} &\le \const{c130} n^{(1-t-p)/2 + \max(0,3/2-q-r) + \max(0,p/2-q-r/2)} \sqrt{\ln(ndk)},
}
for universal positive constant $\const{c130}$.
\end{restatable}

\begin{restatable}{lemma}{differencesurvival} (\textbf{Upper bound on survival variance}): \label{lemma:differencesurvival}
For any fixed competing feature $\beta \in [k]$, $\beta \ne \alpha$ with $\lambda_\a = \lambda_\b$, we have with probability at least $\left(1 - 15/(nk)\right)$,
\eqn{
    && \frac{\hhat_{\a,\b}[\a] - \hhat_{\b,\a}[\b]}{\hhat_{\a,\b}[\a]} &\le \frac{2\const{c50}(\mubar \sqrt{n} \sqrt{\ln(nk)} + \delmu \cdot n/\sqrt{k})}{\const{c60}\mubar\frac{n}{k}\sqrt{\ln(k)}
    - \const{c50}(\mubar \sqrt{n} \sqrt{\ln(nk)} + \delmu \cdot n/\sqrt{k})},
    \label{eqn:differencesurvivalbound}
}
for universal positive constants $\const{c50}$ and $\const{c60}$.
\end{restatable}

As before, we can also obtain the asymptotic bound:
\begin{restatable}{corollary}{asymptoticdifferencesurvival}
\label{corollary:asymptoticdifferencesurvival}
Under the bi-level ensemble model~\ref{def:bilevel}, for any fixed $\beta \in [k]$, $\beta \neq \alpha$, if $t < 1/2$, $t < 2(q+r-1)$, and $1 < q + r < (p + 1)/2$, with probability at least $(1 - 15/(nk))$,
\eqn{
    && \frac{\hhat_{\a,\b}[\a] - \hhat_{\b,\a}[\b]}{\hhat_{\a,\b}[\a]} < n^{-u},
}
for large enough $n$ for some fixed $u > 0$. 
\end{restatable}




Next, we assume that the lemmas and corollaries stated above are true and complete the proof for Theorem~\ref{theorem:regimes}. We provide proofs for these lemmas in Appendices~\ref{app:common_results}, \ref{app:utility_bounds} and \ref{app:misclassification_events}.


Assume we are in the regime where $t < 1/2$, $t < 2(q+r-1)$, and $1 < q + r < (p+1)/2$, so all our corollaries above hold. Denote the misclassification event as $\Eerr$ and let $\eps > 0$ be an arbitrarily chosen constant.


Substitute Corollaries 
\ref{corollary:asymptoticfeaturedifferencebound}, \ref{corollary:asymptoticcontaminationbound}, and \ref{corollary:asymptoticdifferencesurvival} into \eqref{eqn:correctclassificationscenario}, applying them on all $1 \le \beta \ne \alpha \le k$. They hold with probability at least $1-5/(nk)$, $1-7/(nk)$, and $1-15/(nk)$ respectively for a given test point and choice of $\beta$. So by the union bound across the three bounds and all $k-1$ choices of $\beta$, with probability at most $27/n$, one of these corollaries will not hold for our test point for some $\beta$. Let this failure event be denoted $E_1$.


In the case when $E_1$ does not occur, misclassification occurs only if
\eqn{
    \frac{\const{c120}\sqrt{\ln(k)}}{\const{c30}\sqrt{\ln(ndk)}}n^v \left( \min_\b \left( x_{test}[\a] - x_{test}[\b]\right) - \max_\b |x_{test}[\b]| \cdot n^{-u} \right)
   < \max_\beta Z^{(\beta)},
}
where we define the exponent
\eqn{
    && v &= 1-q-r-t - (1-t-p)/2 - \max\left(0,\frac{3}{2}-q-r\right) - \max\left(0, \frac{p}{2}-q-\frac{q}{2}\right) \\
    &&&= \frac{p+1}{2} - q - r - \frac{t}{2} - \max\left(0, \frac{3}{2} - q - r, \frac{p}{2} - q - \frac{r}{2}, \frac{3}{2} - 2q - \frac{3r}{2} \right),
}
and
\eqn{
    Z^{(\beta)} = \frac{1}{\cn_{\a,\b}}\left(\sum_{j \notin \{\a, \b \}} \lambda_j \hhat_{\b,\a}[j] x_{test}[j]  \right).
}

For each class $\b$, observe that we have $Z^{(\beta)} \sim \mathcal{N}(0,1)$.\footnote{To be precise, here we can think of fixing the training data and looking purely at the randomness arising from the features in the test point. The resulting $Z^{(\beta)}$ is a standard normal. Since we are using the union bound in our proof finally, this is sufficient for our purposes.}
Thus, by the Gaussian tail bound, for each $\beta$ with probability at least $(1 - 1/(n\nc))$,
\begin{align}
    Z^{(\beta)} < \sqrt{2\ln(n\nc)}.
\end{align}

So by the union bound over all $k$ classes $\beta$, with probability at least  $\left(1 - 1 / n\right)$, 
\eqn{
    \max_\beta Z^{(\beta)} < \sqrt{2\ln(n\nc)}.
}

Let the failure event where this is not the case be $E_2$. 

An identical argument shows that with probability at least $\left(1 - 2/n\right)$, $\max_\beta \abs{x_{test}[\beta]} \le \sqrt{2\ln(nk)}$. Let $E_3$ be the failure event where this is not the case.

From Lemma \ref{lemma:marginbound}, we know with probability $1 - \eps$ that, if $t > 0$, then for sufficiently large $n$ (and so sufficiently large $k$) 
\eqn{
     \min_\b \left( x_{test}[\a] - x_{test}[\b]\right) > \frac{\theta}{\sqrt{2\ln(k)}}.
}
If $t = 0$ and $k = c_k$, then Lemma \ref{lemma:constantmarginbound} states that, with probability $1 - \eps$,
\eqn{
    \Pr\left(\min_{1 \le \beta \ne \gamma \le c_k} \abs{x_{test}[\beta] - x_{test}[\gamma]} \ge \eps' \right) \ge 1 - \eps,
}
for some constant $\eps'$. Let the $\eps$-probability event of the appropriate margin bound (depending on whether $t = 0$ or $t > 0$) being violated be the error event $E_4$.

Assuming $E_1$, $E_2$, $E_3$, and $E_4$ all do not take place, misclassification can only occur if
\eqn{
    \frac{\const{c120}\sqrt{\ln(k)}}{\const{c30}\sqrt{\ln(ndk)}}n^v \left( \min\left(1-\eps, \frac{\theta}{\sqrt{2\ln(k)}}\right) - \sqrt{2\ln(n\nc)} n^{-u} \right)
   < \sqrt{2\ln(n\nc)}.
}
Clearly, if $v > 0$, then (for sufficiently large $n$) misclassification becomes asymptotically impossible (except via the specified error events), since the LHS of the above grows asymptotically faster than the RHS.

The union bound shows that the probability of any of $E_1, E_2, E_3, E_4$ occurring tends to $\eps$ as $n \to \infty$ (since the probability of the first three tend to zero). So in the regime where
\eqn{
    && t &< \frac{1}{2} \\
    && t &< 2(q+r-1) \\
    && q + r &> 1 \\
    && \frac{p+1}{2} &> q + r + \frac{t}{2} + \max\left(0,\frac{3}{2}-q-r\right) + \max\left(0,\frac{p}{2}-q-\frac{r}{2}\right) \label{eqn:upperboundonh},
}
the probability of misclassification tends to $\eps$ for sufficiently large $n$, for any $\eps > 0$.

Consolidation of the above bounds produces the conditions \footnote{We can simplify~\eqref{eqn:upperboundonh} as follows:
\begin{align}
\frac{p+1}{2} > q + r + \frac{t}{2} & \implies t < p + 1 - 2(q+r) \\
\frac{p+1}{2} > q + r + \frac{t}{2} + \frac{3}{2} - q - r & \implies t < p-2\\
\frac{p+1}{2} > q + r + \frac{t}{2} + \frac{p}{2} -q -\frac{r}{2} & \implies t < 1-r\\
\frac{p+1}{2} > q + r + \frac{t}{2} + \frac{3}{2} - q - r + \frac{p}{2} -q -\frac{r}{2} & \implies t < 2q + r - 2. \end{align}

Then we note that $t< \min(r, 1-r) \implies t < 1/2$. 
}
\begin{align}
    t &<  \min\left(1-r, p+1-2(q+r), p-2, 2q+r-2\right) \\
    q + r &> 1.
\end{align}
Finally, note that the condition $t < r$ comes from the definition of the bi-level model~\eqref{def:bilevel}. This condition simply states that for good generalization we must favor all the features used to determine classes. 
Since the analysis above holds for any $\eps$, we see that within this regime the probability of misclassification must approach zero in the limit.
This completes the proof. Note that while we show that  probability of misclassification goes to zero, we do not show it to do so at any particular rate, because the result from Lemma \ref{lemma:marginbound} does not specify the rate of convergence.

\section{Useful results from elsewhere that we need}
\label{app:common_results}
This section collects results that are used in our proof, but which come from elsewhere or are lightly adapted to our purposes. 



\textbf{Hanson-Wright inequality \citep{rudelson2013hanson}}: Let $\vec{z}$ be a random vector composed of i.i.d.~random variables that are zero mean and with sub-Gaussian norm at most $K$. The sub-Gaussian norm $\|\xi\|_{\psi_2}$ of a random variable $\xi$ is defined as in \citet{rudelson2013hanson},
\begin{align}
    \|\xi\|_{\psi_2} &= \inf_{K > 0} K \\
    & \text{s.t.}  \  \EX\exp\left(\xi^2 / K^2 \right) \leq 2. \label{eq:subgaussiannorm}
\end{align}
Then, there exists universal constant $c > 0$ such that for any positive semi-definite matrix $M$ and for every $t \geq 0$, we have
\begin{align}
    \Pr\left[|\vec{z}^T M \vec{z} - \EX[\vec{z}^T M \vec{z} ]| > t \right] \leq 2 \exp\left\{-c \min\left\{\frac{t^2}{K^4||M||_{\mathsf{F}}^2}, \frac{t}{K^2||M||_{\mathsf{op}}}\right\}\right\} \label{eq:hansonwright}
\end{align}

The next result bounds the eigenvalues of the $n \times n$ matrix $\Abold = \Xw (\Xw) \tran$, where recall that the rows of $\Xw$ are i.i.d. Gaussians with covariance matrix $\boldsymbol{\Lambda} = \mathrm{diag}(\lambda_1, \lambda_2, \dots, \lambda_d)$. Let $\mu_1(\Abold)$ denote the largest eigenvalue and $\mu_n(\Abold)$ denote the smallest eigenvalue of $\Abold$ respectively. 



From \citet{reg:bartlett2020benign}\footnote{More precisely this lemma appeared in the first version of this work at \url{https://arxiv.org/pdf/1906.11300v1.pdf}. In subsequent versions the authors use a slightly weaker version of this result since it is sufficient for their purpose.}, we have the following result 

\begin{lemma}
With probability at least $(1 - 2e^{-n})$, the eigenvalues of $\Abold$ satisfy:

\begin{align}
   \sum_j \lambda_j - \Diamond \leq \mu_n(\Abold) \leq \mu_1(\Abold) \leq \sum_j \lambda_j + \Diamond,
   \label{eq:eigvalbound}
\end{align}
where,
\begin{align}
    \Diamond = \frac{32}{9} \left( \lambda_1 (1 + \ln 9) n   + \sqrt{(1 + \ln 9) n \sum_j \lambda_j^2} \right). \label{eq:diamond2}
\end{align}
\end{lemma}

Next, as stated previously in Lemma \ref{lemma:eigenvaluebounds} we will use this result to obtain bounds on the eigenvalues of $\Ainvbold$ assuming that $\boldsymbol{\Lambda}$ is such that 
$\Diamond \ll \sum_j \lambda_j$.\footnote{Note that in the regime $q+r<1$ (where regression works \citep{binary:Muth20}), we do not have $\Diamond \ll \lambda_j$ and in such scenarios we cannot simply rely on eigenvalue bounds and need to use other techniques in the proof.}



\eigenvaluebounds*

\begin{proof}
Let $S =\sum_j \lambda_j$.
\begin{align}
    \frac{1}{S + \Diamond} &= \frac{1}{S} \left( 1 + \frac{\Diamond}{S} \right)^{-1} \\
    &= \frac{1}{S} \left(1 - \frac{\Diamond}{S}  + \Theta\left( \frac{\Diamond}{S}  \right)^2 \right) \\
    &= \mubar - \delmu,
\end{align}
and analogously $(S-\Diamond)^{-1} = \mubar + \delmu$. Taking reciprocals of everything in the inequality~\ref{eq:eigvalbound}, and since the eigenvalues of $\Abold$ and $\Ainvbold$ are  reciprocals of each other, the desired result follows.

\end{proof}

As a Corollary of Lemma \ref{lemma:eigenvaluebounds}:
\begin{corollary}
\label{corollary:delmuscaling} (\textbf{Asymptotic eigenvalue bounds on $\Ainvbold$}) 
Considering the asymptotic scaling of the model parameters from the bi-level model (Definition~\ref{def:bilevel}), in the regime $1 < q + r < (1+p)/2$,
\eqn{
    && \mubar &= n^{-p} \\
    && \delmu &\le \const{c80} n^{1-p-q-r} \ll \mubar,
}
where $\mubar$ and $\delmu$ are defined as in Lemma \ref{lemma:eigenvaluebounds}, and \const{c80} is a universal constant.
\end{corollary}

\begin{proof}
From the asymptotic scaling of the $\lambda_j$ from \eqref{eq:bilevelparamscaling} and \eqref{eq:bilevellambdascaling}, we see that (from the definition provided in Lemma \ref{lemma:eigenvaluebounds})
\eqn{
    && \mubar &= \frac{1}{\sum_j \lambda_j} \\
    &&&= \frac{1}{n^r n^{p-q-r} + (n^p - n^r)(1 - n^{q})\cdot n^p/(n^p - n^r))} \\
    &&&= \frac{1}{n^{p-q} + n^p - n^{p-q}} \\
    &&&= n^{-p}.
}

Next, we have that
\eqn{
    \Diamond &= \frac{32}{9} \left( \lambda_1 (1 + \ln 9) n   + \sqrt{(1 + \ln 9) n \sum_j \lambda_j^2} \right) \\
    &\le \const{c90} n^{1+p-q-r} + \const{c100}\sqrt{n(n^rn^{2p-2q-2r} + (n^p-n^r))} \\
    &\le \const{c90} n^{1+p-q-r} + \const{c100}\sqrt{n^{1+2p-2q-r} + n^{1+p}}
}
for constants \newc\label{c90} and \newc\label{c100},

The second term is of the order $n^{\max((1-r)/2 + p-q, (1+p)/2)}$. Thus, in the regime
 $q + r < (1+p)/2$, and since $r < 1$ we have $1+p-q-r >(1-r)/2 + p-q$ and $1+p-q-r>  (1+p)/2$ and the first term dominates.
 
 Thus, $\Diamond \le \const{c110} n^{1+p-q-r}$ for some constant \newc\label{c110} and sufficiently large $n$.

Observe that since $q+r >1$, $\Diamond \ll \sum_j \lambda_j = n^p$. 
Thus, we can substitute into our relation for $\delmu$ from Lemma \ref{lemma:eigenvaluebounds}, to see that
\eqn{
    && \delmu &= \mubar \left( \frac{\Diamond}{\sum_j \lambda j} + \Theta\left( \frac{\Diamond}{\sum_j \lambda j} \right)^2 \right) \\
    &&&\le n^{-p} \left((\const{c110} n^{1+p-q-r})(n^{-p}) + \Theta((\const{c110} n^{1+p-q-r})^2(n^{-p})^2)\right) \\
    &&&= n^{-p}(\const{c110} n^{1-q-r} + \Theta(\const{c110} n^{2(1-q-r)})).
}

In the regime where $q + r > 1$, the first term in the sum dominates the second, giving us,
\eqn{
    && \delmu &\le \const{c80} n^{1-p-q-r}
}
for some constant \newc\label{c80} and sufficiently large $n$. This completes the proof.
\end{proof}

Finally, in this section, we restate well-known bounds concerning Gaussian random variables.

\begin{lemma} \label{lemma:chisquared}
Chi-squared tail bound:\\
Let $\mathbf{z} \sim \mathcal{N}(0,I_n)$.
For any $\delta \in (0, 1)$, with probability at least $(1 - 2e^{-n\delta^2})$ we have:
\begin{align}
n(1-\delta) \leq \| \mathbf{z} \|^2 \leq n(1 + \delta). \label{eq:znormbound}
\end{align}
\end{lemma}

From bounds on the expectation of the maximum of $k$ Gaussians:
\begin{lemma}
\label{lemma:maxofkgaussians}
Let $\za = \max_{1 \leq j \leq \nc} \zj$ where $\zj \sim \mathcal{N}(0, 1)$. Then,
\begin{align}
  \frac{1}{\sqrt{\pi \ln 2}} \cdot \sqrt{\ln \nc} \leq  \EE[ \za] \leq \sqrt{2} \cdot \sqrt{\ln \nc}. \label{eq:maxgaussianmeanbound}
\end{align}
\end{lemma}

\section{Utility Bounds}
\label{app:utility_bounds}

The big technical challenge in moving from binary classification (as studied in \citet{binary:Muth20}) to multiclass classification has to do with the nature of the training data. Whereas for binary classification one could change coordinates so that the binary labels only depended on a single Gaussian random variable and were independent of all other directions of Gaussian variation in the covariates, no such change of coordinates exists for multiclass labels. The one-hot-style encoding of the labels fundamentally depends on the realizations of all $k$ of the Gaussian random variables representing each of the $k$ classes. This means that we can no longer simply leverage independence to simplify the analysis and certain clever approaches used to invoke Hanson-Wright are no longer available to us. However, the need remains to appropriately bound quadratic forms of the form $\abs{\zj \tran \Ainv \dely}$ both for the cases when $j$ represents a feature that is not dominant in the computation of $\dely$ as well as in cases where $j$ represents a feature that is dominant in $\dely$. 
To be able to control such quantities in the absence of the independence we could leverage in the binary case, this section of the Appendix derives two lemmas which can be viewed as helper bounds. These bounds will later be used to bound the various quantities from \eqref{eqn:correctclassificationscenario}. Because our focus is on the asymptotic scaling, we will use $c_i$ to denote the appropriate global constants.

In the subsequent lemmas, $\mubar$ and $\delmu$ are defined as in the bounds on the eigenvalues of $\Ainv$ from Lemma \ref{lemma:eigenvaluebounds}.

The following lemma is used to upper-bound the contamination term $\cn_{\alpha,\beta}$ in Lemma \ref{lemma:contaminationbound}:
\begin{restatable}{lemma}{zdeltaybound}
\label{lemma:zdeltaybound}
Let $\dely = \vec{y}_\alpha - \vec{y}_\beta$. Let $\alpha$, $\beta$, and $j$ be distinct. Then, with probability at least $\left(1 - 7/(ndk)\right)$, we have,
\eqn{
    \abs{\zj \tran \Ainv \dely} \leq \const{c30}(\mubar \sqrt{\frac{n}{\nc}}\cdot \sqrt{\ln(nd\nc)} + \delmu \cdot n/\sqrt{k}),
}
for some constant $\const{c30}$.
\end{restatable}

This next lemma is used to bound the numerator of the survival variation term from \eqref{eqn:correctclassificationscenario}:
\begin{restatable}{lemma}{zdiffybound}
\label{lemma:zdiffybound}
Let $\dely = \vec{y}_\alpha - \vec{y}_\beta$. With probability at least $\left(1 - 5/(nk)\right)$, we have each of
\eqn{
    \za \tran \Ainv \dely &\le \mubar(\EE[ \za\tran \ya] -\EE[ \za\tran \yb]) + \const{c50}(\mubar \sqrt{n} \sqrt{\ln(nk)} + \delmu \cdot n/\sqrt{k}) \label{eqn:zdiffyboundupper} \\
    \za \tran \Ainv \dely &\ge \mubar(\EE[ \za\tran \ya] -\EE[ \za\tran \yb]) - \const{c50}(\mubar \sqrt{n} \sqrt{\ln(nk)} + \delmu \cdot n/\sqrt{k}), \label{eqn:zdiffyboundlower}
}
for some constant \const{c50}.
\end{restatable}

The following corollary of the above is used to lower-bound the relative survival $\hhat_{\a,\b}[\a]$, which in turn bounds the SU/CN ratio and the denominator of the survival variation term:
\begin{restatable}{corollary}{zdiffyboundnoexp}
\label{corollary:zdiffyboundnoexp}
Let $\dely = \vec{y}_\alpha - \vec{y}_\beta$. With probability at least $\left(1 - 5/(nk)\right)$, we have,
\eqn{
    \za \tran \Ainv \dely \ge \const{c60}\mubar\frac{n}{k}\sqrt{\ln(k)}
    - \const{c50}(\mubar \sqrt{n} \sqrt{\ln(nk)} + \delmu \cdot n/\sqrt{k}),
}
for some constant $\const{c60}$. 
\end{restatable}

\subsection{Proof of Lemma \ref{lemma:zdeltaybound}}
\label{app:prooflemmazdeltaybound}

We will write $\Ainv = \mubar \In + \Ainvdel$, and split up the expression $\zj \tran \Ainvbold \dely$ into components involving $ \mubar \In$, and components involving $\Ainvdel$. To bound the first term, we will use Hanson-Wright, and to bound the second we will use Cauchy-Schwartz.  Throughout the proof, we rely on the concentration of the eigenvalues of $\Ainv$.

Next, we bound the first term (we set aside the constant $\mubar$ for now and deal with it later).
\subsubsection{Bounds on $\vec{z}_j^T(\vec{y}_\alpha - \vec{y}_\beta)$}
\label{sec:introduceMbold}


Throughout this section, let $j$ be a feature index distinct from $\a$ and $\b$.
Define the diagonal matrix $\Mbold \in \mathbb{R}^{n \times n}$ with diagonal entries given by:
\begin{align}
    \Mii = \begin{cases} 1,& \mathrm{if} \ \delyi \neq 0 \\
    0, & \mathrm{otherwise}\end{cases}.
\end{align}
In other words, $M_{ii}$ is $1$ only if training point $i$ belongs to class $\alpha$ or $\beta$ and is $0$ otherwise. Thus for each $i \in [n]$,  $M_{ii} \sim Bernoulli(2/k)$ and are independent of each other. We introduce this matrix $\Mbold$ to ensure that our bound reflects the fact that most of the entries of $\dely$ are 0. In particular $\delyi \ne 0$ only if point $i$ belongs to class $\alpha$ or $\beta$ and only contains roughly $2n/k$ non-zero entries.\footnote{An alternative bounding technique that first converted $\zj \tran \dely$ to a quadratic form and applied Hanson-Wright would be looser by a factor of $\sqrt{k}$ if we did not introduce $\Mbold$.} 
Note that we have by definition,
\eqn{
    \vec{z}_j^T \dely &=  \vec{z}_j^T \Mbold\dely.
}
Our strategy is to bound $\zj \tran \Mbold \dely$  for every typical realization $\Mtypical$ of the random variable $\Mbold$ using the Hanson-Wright inequality. Subsequently, we will apply these bounds with high probability over typical realizations of $\Mbold$ that satisfy the Proposition below, which merely asserts that with high probability, the number of $1$s in $\dely$ is close to its expected value.
\begin{proposition}
\label{proposition:Mfrobnormbound}
For $\delta \in (0,1)$, with probability at least $(1 - 2e^{-\frac{2n\delta^2}{3\nc}})$, the trace of $\Mbold$ is bounded as:
\begin{align}
    (1 - \delta) \frac{2n}{\nc}\leq \norm{\dely}_2^2 = \Tr{\Mbold} \leq (1 + \delta) \frac{2n}{\nc} \label{eqn:Mfrobnormbound}.
\end{align}
\end{proposition}

\begin{proof}
Note that $\Tr{\Mbold}$ is the sum of $n$ i.i.d Bernoulli random variables with mean $2/k$. The result follows by application of the Chernoff bound.
\end{proof}

Note that once we fix the realization $\Mtypical$, the distributions of $\zj$ and $\dely$ will now have to be conditioned on this realization and we need to deal with the modified distributions while applying the Hanson-Wright inequality. In particular, once we know that a feature was not the winning feature, it is no longer zero-mean.

Now,
\begin{align}
     \vec{z}_j^T \Mtypical \dely &= \sum_{i} \zji \Mtii \delyi \\
     &= \sum_{i: \Mtii=1}\zji\delyi\\
     &= \sum_{i: \Mtii=1}\left(\zji - \EX[\zji \mid \Mii = 1]\right)\delyi + \sum_{i: \Mtii=1} \EX[\zji \mid \Mii = 1] \delyi \\
     &= \sum_{i : \Mtypical_{ii} =  1} \tilde{z}_{j,\Mtypical}[i] \delyi + \sum_{i: \Mtii=1} \EX[\zji \mid \Mii = 1] \delyi \label{eq:splittwotermsMmagic}, 
\end{align}
where $\tilde{z}_{j,\Mtypical}[i]$ is now a zero-mean random variable conditioned on the realization $\Mtypical$.

First, we bound the term $\sum_{i : \Mtypical_{ii} =  1} \tilde{z}_{j,\Mtypical}[i] \delyi$. We collect the elements corresponding to indices where $\Mtypical_{ii} = 1$ into the vectors $\vec{z}'_{j, \Mtypical}$ and $\dely'_{ \Mtypical}$, which are both length $\Tr{\Mtypical}$ (Figure~\ref{fig:primeillustration} shows an example of collecting elements). 

\begin{figure}[h!]

\[
   \underbrace{\mat{1 \\ 2 \\ 3 \\ 4 \\ 5}}_{\tilde{\vec{z}}_{j, \Mtypical}}, \underbrace{\mat{1 \\ 0 \\ -1 \\ 1 \\ 0}}_{\dely} \to \underbrace{\mat{1 \\ 3 \\ 4}}_{\vec{z}'_{j, \Mtypical}}, \underbrace{\mat{1 \\ -1 \\ 1}}_{\dely'_\Mtypical}
\]
\caption{An example of collecting elements at indices where $\Mtypical_{ii} = 1$ into smaller vectors of length $\Tr{\Mtypical}$. Recall that $\dely[i] \ne 0$ iff $\Mtypical_{ii} = 0$.}
 \label{fig:primeillustration}
\end{figure}

We can then express
\begin{align}
    \sum_{i : \Mtypical_{ii} =  1}& \tilde{z}_{j,\Mtypical}[i] \delyi\\
     & = (\vec{z}'_{j, \Mtypical})^T \dely'_{\Mtypical} \\
     & = \frac{1}{4} \left((\vec{z}'_{j, \Mtypical} + \dely'_{\Mtypical})^T \mathbf{I}_{\Tr{\Mtypical}} (\vec{z}'_{j, \Mtypical} + \dely'_{\Mtypical}) -(\vec{z}'_{j, \Mtypical} - \dely'_{\Mtypical})^T \mathbf{I}_{\Tr{\Mtypical}} (\vec{z}'_{j, \Mtypical} - \dely'_{\Mtypical})\right) \label{eqn:zdeltayhw},
\end{align}
where we added and subtracted terms in the last equality. 

We  prove via the subsequent propositions that conditioned on the realization $\Mtypical$, the entries of $\vec{z}'_{j, \Mtypical} \pm \dely'_{j, \Mtypical}$ are i.i.d. and sub-Gaussian with bounded norm. Thus, they satisfy the requirements to apply the Hanson-Wright inequality from \citet{rudelson2013hanson} to bound the two quadratic forms in the above expression \eqref{eqn:zdeltayhw}.

\begin{proposition}
\label{proposition:subgaussianconditioned}
Conditioned on the realization $\Mtypical$, $z'_{j,\Mtypical}[i']$ has sub-Gaussian norm at most $6$.
\end{proposition}
\begin{proof}

Let $i$ be the original index from which $z'_{j,\Mtypical}[i']$ was sampled. 



If $j > k$, then $z'_{j,\Mtypical}[i']= \tilde{z}_{j, \Mtypical}[i] = \zji$ irrespective of the realization $\Mtypical$ because feature $j$ is not used in the comparison to determine the class label and is independent to  $y_\alpha$ and $y_\beta$ (and thus independent to $\Mbold$). Further, $\zji$ is simply a Gaussian (and therefore sub-Gaussian with sub-Gaussian norm $\|z_j[i]\|_{\psi_2} \le 2$. Here we use the definition of sub-Gaussian norm from~\eqref{eq:subgaussiannorm} reproduced here for convenience:

The sub-Gaussian norm of a random variable $\xi$ is given by,
\begin{align}
    \|\xi\|_{\psi_2} &= \inf_{K > 0} K \\
    & \text{s.t.}  \  \EX\exp\left(\xi^2 / K^2\right) \leq 2.
\end{align}


Otherwise, if $j$ is one of the $k$ features that define classes, since
\eqn{
    z'_{j,\Mtypical}[i'] &= \tilde{z}_{j,\Mtypical}[i] \\
    &= z_{j}[i] - \EX[z_{j}[i] \mid M_{ii} =1],
}
the triangle inequality states that
\eqn{
    \| \tilde{z}_{j,\Mtypical}[i] \|_{\psi_2} &\le \| z_{j}[i] \|_{\psi_2} + \|\EX[z_{j}[i] \mid M_{ii}=1] \|_{\psi_2}.
}
Note that the distribution of $z_j[i]$ conditioned on realization $\Mtypical$ is equivalent to the distribution obtained by conditioning on the event $\Mii =1$. So it is sufficient to compute these sub-Gaussian norms conditioned on the event $\Mii = 1$.

We will first bound $ \|z_{j}[i] \|_{\psi_2}$. Let $\maxevent_j$ be the event that $z_j[i]$ is the maximum out of the first $k$ features, and let $\nonmaxevent_j$ be the complementary event. 
 

First, without conditioning on $\maxevent_j$, we know by well-known results for the standard Gaussian that
\begin{align}
    \EE \exp(\vec{z}_j[i]^2/5) = \sqrt{\frac{5}{3}} \leq \frac{4}{3}.
\end{align}

Using the law of iterated expectation we can relate this to the expectation conditioned on the events $\maxevent_j$ and $\nonmaxevent_j$, noting that $P(\maxevent_j) = 1/k$:
\begin{align}
    \frac{4}{3} &\ge \EE \exp(\vec{z}_j[i]^2/5) \\ 
    &= P(\maxevent_j) \EE \exp(\vec{z}_j[i]^2/5 | \maxevent_j) +  P(\nonmaxevent_j) \EE \exp(\vec{z}_j[i]^2/5 | \nonmaxevent_j) \\
    &= \frac{1}{k}\EE \exp(\vec{z}_j[i]^2/5 | \maxevent_j)  + \frac{k-1}{k} \EE \exp(\vec{z}_j[i]^2/5 | \nonmaxevent_j). 
\end{align}

Rearranging terms, we obtain,
\begin{align}
    \frac{k-1}{k} \EE \exp(\vec{z}_j[i]^2/5 | \nonmaxevent_j) &\leq \frac{4}{3} - \frac{1}{k}\EE \exp(\vec{z}_j[i]^2/5 | \maxevent_j) \\
    \implies \EE \exp(\vec{z}_j[i]^2/5 | \nonmaxevent_j) &\leq \frac{k}{k-1} \left( \frac{4}{3} - \frac{1}{k}\EE \exp(\vec{z}_j[i]^2/5 | \maxevent_j) \right) \\
    &\leq \frac{k}{k-1} \cdot \frac{4}{3} \\
    & \leq {2},
\end{align}
where in the second to last inequality we used the non-negativity of $\EE \exp(\vec{z}_j[i]^2/5 | \maxevent_j)$ and in the last equality we assumed $k \geq 3$.
We then have
\eqn{
    && \EE \exp(\vec{z}_j[i]^2/5 | \nonmaxevent_j) &= \sum_{m \ne j} \EE \exp(\vec{z}_j[i]^2/5 | \nonmaxevent_j \cap \maxevent_m) P(\maxevent_m \mid \nonmaxevent_j) \\
    &&&= \frac{1}{k-1} \sum_{m \ne j} \EE \exp(\vec{z}_j[i]^2/5 | \nonmaxevent_j \cap \maxevent_m) \label{eqn:miiexp1}
}
where the last equality follows by symmetry. Further by symmetry, all the terms in the above summation that we are averaging are equal, so we can express it as an average of just the terms corresponding to $m = \a$ and $m = \b$, as follows:
\eqn{
    && \eqref{eqn:miiexp1} &= \frac{1}{2} \EE \exp(\vec{z}_j[i]^2/5 | \nonmaxevent_j \cap \maxevent_\alpha) + \frac{1}{2} \EE \exp(\vec{z}_j[i]^2/5 | \nonmaxevent_j \cap \maxevent_\beta) \\
    &&&= P(\maxevent_\alpha \mid \nonmaxevent_j \cap (\maxevent_\a \cup \maxevent_\b )) \EE \exp(\vec{z}_j[i]^2/5 | \nonmaxevent_j \cap \maxevent_\alpha) \nonumber \\
    &&& \quad \quad +P(\maxevent_\beta \mid \nonmaxevent_j \cap (\maxevent_\a \cup \maxevent_\b ))\EE \exp(\vec{z}_j[i]^2/5 | \nonmaxevent_j \cap \maxevent_\beta)  \label{eqn:miiexp2},
}
again by symmetry. Since exactly one of $\maxevent_\a$ and $\maxevent_\b$ are true when conditioned on $\nonmaxevent_j \cap (\maxevent_\a \cup \maxevent_\b )$, we can rewrite the above as our desired expectation
\eqn{
    && \eqref{eqn:miiexp2} &= \EE \exp(\vec{z}_j[i]^2/5 | \nonmaxevent_j \cap (\maxevent_\alpha \cup \maxevent_\beta)) \\
    &&&= \EE \exp(\vec{z}_j[i]^2/5 | \maxevent_\alpha \cup \maxevent_\beta) \\
    &&&= \EE \exp(\vec{z}_j[i]^2/5 | M_{ii} = 1),
}
since $M_{ii} = 1$ is equivalent to the event $\maxevent_\alpha \cup \maxevent_\beta$. Thus, conditioned on the event $M_{ii} = 1$, $\norm{{z}_j[i]}_{\psi_2} \leq \sqrt{5}$.

Next we consider $\|\EX[z_{j}[i] \mid M_{ii}= 1] \|_{\psi_2}$. By a similar argument to above, we have that $\EX[z_j[i] \mid M_{ii} = 1] = \EX[z_j[i] \mid \nonmaxevent_j]$, so we will focus on the second quantity instead. Bounds on the max of Gaussians (Lemma \ref{lemma:maxofkgaussians}) state that:
\eqn{
    && 0<\EX[\vec{z}_j[i] \mid \maxevent_j] &\le \sqrt{2\log(k)} \\
    \thus 0 > \EX[\vec{z}_j[i] \mid \nonmaxevent_j] &\ge -\frac{1}{k-1} \sqrt{2\log(k)} \ge -2 \\
    \thus \exp\left(\frac{\EX[\vec{z}_j[i] \mid \nonmaxevent_j]^2}{3^2}\right) &< 2.
}
In the second last inequality we use the fact that the function $f(k) = \abs{\sqrt{2\log k}/(k-1)}$ is monotonically decreasing in $k$ and assumed $k \geq 3$.

Thus, the (constant) random variable $\EX[\vec{z}_j[i] \mid M_{ii} = 1]$ is sub-Gaussian with parameter $3$. So, by the triangle inequality, conditioned on $\Mii=1$
\eqn{
    && \norm{\tilde{{z}}_{j,m}[i]}_{\psi_2} &\le \norm{{z}_j[i]}_{\psi_2} + \norm{\EX[\tilde{{z}}_{j,m}}_{\psi_2} \\
    &&&\le \sqrt{5} + 3 \\
    &&&\le 6.
}
This completes the proof that conditioned on the realization $\Mtypical$, $z'_{j,\Mtypical}[i']$ is  sub-Gaussian with norm at most 6.
\end{proof}

We can now prove our target result:
\begin{proposition} \label{prop:zdeltay}
With probability at least $\left(1 - 6/(ndk)\right)$,
\eqn{
 \abs{\zj\tran \dely}  \leq \const{c20}  \sqrt{\frac{n}{\nc}} \cdot \sqrt{\log(nd\nc)}.
}

for universal constant $\const{c20}$.
\end{proposition}

\begin{proof}
Our strategy will be to bound $\zj\tran \dely = \zj \tran \Mbold \dely$ for every typical realization $\Mtypical$ of $\Mbold$ that satisfies Proposition \ref{proposition:Mfrobnormbound}.  Recall that for a given realization $\Mtypical$ we have,
\eqn{
    \vec{z}_j^T \Mtypical \dely &= \sum_{i : \Mtypical_{ii} =  1} \tilde{z}_{j,\Mtypical}[i] \delyi + \sum_{i: \Mtii=1} \EX[\zji \mid \Mii = 1] \delyi. \label{eq:jcrosstermsplit}
}
We will use Hanson-Wright to bound the first term, which we previously expressed in \eqref{eqn:zdeltayhw} as:
\eqn{
    && \sum_{i : \Mtypical_{ii} =  1}& \tilde{z}_{j,\Mtypical}[i] \delyi \\
    &&&= \frac{1}{4}\left((\vec{z}'_{j, \Mtypical} + \dely'_{\Mtypical})^T \mathbf{I}_{\Tr{\Mtypical}} (\vec{z}'_{j, \Mtypical} + \dely'_{\Mtypical}) -(\vec{z}'_{j, \Mtypical} - \dely'_{\Mtypical})^T \mathbf{I}_{\Tr{\Mtypical}} (\vec{z}'_{j, \Mtypical} - \dely'_{\Mtypical})\right).
}
By Proposition \ref{proposition:subgaussianconditioned}, the sub-Gaussian conditions for the entries of $\vec{z}'_{j, m}$ are satisfied. Further, $\dely'_{\Mtypical}$ is bounded in $[-1, 1]$, so $\|\dely'_{\Mtypical}\|_{\psi_2} \le 2$. Thus, by the triangle inequality, the sub-Gaussian norm of the entries of $\vec{z}'_{j, \Mtypical} \pm \dely'_{\Mtypical}$ is bounded by $K \le 6 + 2 = 8$. Also note that  conditioned on the realization $\Mtypical$, $\vec{z}'_{j, \Mtypical}$ is zero-mean by construction and $\dely'_{\Mtypical}$ is zero-mean by symmetry between $\a$ and $\b$, so we can now apply the Hanson-Wright inequality to both terms.

We choose parameter
\eqn{
    t = \frac{K^2}{\sqrt{c}} \sqrt{\Tr{\Mtypical}} \sqrt{\log(ndk)}.
}
where $c$ is the constant from the Hanson-Wright result.

So
\begin{align}
    \frac{t^2}{K^4\|\mathbf{I}_{\Tr{\Mtypical}}\|_{\mathsf{F}}^2} &= \frac{1}{c} \log(nd\nc) \\
    \frac{t}{K^2\|\mathbf{I}_{\Tr{\Mtypical}}\|_{\mathsf{op}}} &= \frac{1}{\sqrt{c}} \sqrt{\Tr{\Mtypical}} \sqrt{\log(nd\nc)} > \frac{1}{c}\log(nd\nc).
\end{align}

The last inequality follows since with high probability  $\Tr{\Mtypical} = \Theta(\sqrt{n/k})$, by Proposition \ref{proposition:Mfrobnormbound},  $\sqrt{\Tr{\Mtypical}}\sqrt{\log(ndk)} = \Theta(\sqrt{n\log(ndk)/k})$ grows faster than $\log(nd\nc)$.

Finally, note that:
\begin{align}
    \EX[(\vec{z}'_{j, \Mtypical})^T \dely'_{\Mtypical} \mid \Mbold = \Mtypical] &= \sum_{i : \Mtypical_{ii} =  1} \EX[\tilde{z}_{j,\Mtypical}[i] \delyi \mid \Mbold = \Mtypical] \\
     &= \sum_{i: \Mtii = 1} \EX[\tilde{z}_{j,\Mtypical}[i] \delyi \mid \Mii=1] \\
     &= \sum_{i: \Mtii = 1} \frac{1}{2}\EX[\tilde{z}_{j,\Mtypical}[i] \mid \delyi = 1] - \frac{1}{2}\EX[\tilde{z}_{j,\Mtypical}[i] \mid \delyi = -1] \\
     &= 0,
\end{align}
where the last equation follows by symmetry. Knowing which of  $\z_{\a}[i]$ or $\z_{\b}[i]$ was the maximum does not change the conditional expectation of $\tilde{z}_{j, \Mtypical}[i]$.

So, applying Hanson-Wright, with probability at least $(1 - 4/(nd\nc))$ we have
\begin{align}
    -\frac{K^2}{2}\const{c10} \sqrt{\Tr{\Mtypical}} \sqrt{\log(nd\nc)} = -\frac{t}{2} \leq \tilde{\vec{z}}_{j,m}^T \Delta \vec{y}  \leq \frac{t}{2} = \frac{K^2}{2}\const{c10} \sqrt{\Tr{\Mtypical}} \sqrt{\log(nd\nc)},
\end{align}
where $\newc\label{c10} = \frac{1}{\sqrt{c}}$.

We next consider the second term $\sum_{i: \Mtii=1} \EX[\zji \mid \Mii = 1] \delyi$ from~\eqref{eq:splittwotermsMmagic} conditioned on the realization $\Mtypical$.
By an identical symmetry argument as for the previous term we have, $0 \ge \EX[z_j[i] \mid \nonmaxevent_j] = \EX[z_j[i] \mid M_{ii} = 1]$. Then as a consequence of Lemma~\ref{lemma:maxofkgaussians} and using the fact that $M_{ii}=1$ implies $z_j[i]$ is not the maximum of $k$ Gaussians we have,
$\EX[z_j[i] \mid \nonmaxevent_j] \ge -2\sqrt{\log(k)}/(k-1)$. So we can bound
\eqn{
    && \abs{\sum_{i: \Mtii=1} \EX[\zji \mid \Mii = 1] \delyi} &\le \frac{2\sqrt{\log(k)}}{k-1} \abs{\sum_{i:\Mtypical_{ii}=1}  \dely_i}     \le \frac{2\delta'\sqrt{\log(k)}}{k-1}, \label{eq:expzdeltaybound}
}
with probability $1 - 2e^{-\delta'^2/\left(6\cdot\Tr{\Mtypical}\right)}$, by application of the Chernoff bound and using the fact that conditioned on $\Mii=1$, $\delyi$ takes value $\pm$1 with probability half by symmetry among features $\a$ and $\b$.

Next, we apply the high probability bounds above on typical realizations $\Mtypical$. In particular, we substitute bounds on $\Tr{\Mbold}$ from \eqref{eqn:Mfrobnormbound} from Proposition \ref{proposition:Mfrobnormbound} with $\delta = 1/2$ into \eqref{eq:expzdeltaybound}, and set $\delta' = \sqrt{6(1+\delta)(n/k)\log(ndk)}$. Then $e^{-\delta'^2/\left(6\cdot\Tr{\Mbold}\right)} \le 1/(ndk)$ and $e^{-\frac{2n\delta^2}{3\nc}} < 1/(nd\nc)$, so using the union bound we have with probability at least $(1 - 4/(nd\nc) - 1/(nd\nc) - 1/(nd\nc))$,
\eqn{
    && \abs{\vec{z}_j^T \Delta \vec{y}} &\le
    \abs{\sum_{i : \Mtypical_{ii} =  1} \tilde{z}_{j,\Mtypical}[i] \delyi}+ \abs{\sum_{i: \Mtii=1} \EX[\zji \mid \Mii = 1] \delyi} \\
    &&&\leq \frac{K^2}{2}\const{c10}\sqrt{1+\delta} \cdot \sqrt{\frac{2n}{\nc}} \cdot \sqrt{\log(nd\nc)} + \frac{2\sqrt{(1+\delta)(n/k)\log(ndk)}    \sqrt{\log(k)}}{k-1} \\
    &&&\leq \frac{K^2}{2}\const{c10}\sqrt{1+\delta} \cdot \sqrt{\frac{2n}{\nc}} \cdot \sqrt{\log(nd\nc)} + \frac{2\sqrt{(1+\delta)}    \sqrt{\log(k)}}{k-1} \cdot \sqrt{\frac{n}{k}} \cdot \sqrt{\log(ndk)} \\
    &&&\leq \const{c20}  \sqrt{\frac{n}{\nc}} \cdot \sqrt{\log(nd\nc)}, \label{eq:zjtrandelybound}
}
for a suitable choice of $\newc\label{c20}$.
\end{proof}

\subsubsection{Bounds on $\zj \tran \Ainv (\ya - \yb)$}
We can now prove bounds on our target quantity. We restate the lemma that we are trying to prove below for convenience. 
\zdeltaybound*

\begin{proof}
We can rewrite
\begin{align}
     \zj \tran \Ainv \dely &= \zj \tran \left(\mubar \In + \Ainvdel \right) \dely \\
    &= \mubar \zj \tran \dely + \zj \tran \Ainvdel \dely.
\end{align}

Next we can bound $\abs{\zj \tran \Ainvdel \dely}$ simply as
\begin{align}
    | \zj \tran \Ainvdel \dely | &\leq \|\zj \|_2 \| \Ainvdel \dely \|_2 \label{eqn:cauchyschwartzlooseness} \\
   &\leq  \| \Ainvdel \|_{op} \| \zj \|_2 \| \dely \|_2\\
   &\leq \delmu \| \zj \|_2 \| \dely \|_2,
\end{align}
where we use the fact that $\Ainvdel$ is a symmetric matrix and its 2-norm is its maximum absolute eigenvalue. We obtain the eigenvalue bounds for $\Ainvdel$ from Lemma \ref{lemma:eigenvaluebounds}, holding with probability at least $1 - 2e^{-n}$.

So, by the triangle inequality, we have with probability at least $(1 - 6/(nd\nc) - 2e^{-n} - 2e^{-\frac{2n\delta^2}{3\nc}} - 2e^{-n\delta^2})$
\begin{align}
     \abs{\zj \tran \Ainv \dely} &\leq \const{c20} \mubar \sqrt{\frac{n}{\nc}}\cdot \sqrt{\ln(nd\nc)} + \delmu \cdot \sqrt{(1 + \delta) n} \cdot \sqrt{(1 + \delta) \frac{2n}{\nc}}.
\end{align}
The first term follows from Proposition \ref{prop:zdeltay}, and the second from our bound on $\Tr{\Mbold} = \norm{\dely}_2^2$ from Proposition \ref{proposition:Mfrobnormbound}, as well as an analogous application of the chi-squared bound (Lemma \ref{lemma:chisquared}) on $\norm{\vec{z}_j}_2$.

The proof follows by setting $\delta$ to any value in $(0, 1)$, choosing an appropriate constant $\newc\label{c30}$, and noting that for large enough $n$, $1/(nd\nc) \gg d_1 e^{-d_2 n/k}$ for any positive constants $d_1, d_2$. 
\end{proof}

\subsection{Proof of Lemma \ref{lemma:zdiffybound}}

\label{app:proofzdiffybound}
Next we use a similar technique as in Appendix~\ref{app:prooflemmazdeltaybound} to bound $\za \tran \Ainvbold \dely$. We will write $\Ainvbold = \mubar \In + \Ainvdel$, and split up the expression $\za \tran \Ainvbold \dely$ into components involving $ \mubar \In$, and components involving $\Ainvdel$.


\begin{proposition} \label{proposition:zybound}
Consider two arbitrary length-$n$ zero-mean vectors $\vec{y}$ and $\vec{z}$ whose components each has sub-Gaussian norm at most $K$. With probability at least $1 -4/(n\nc)$ we have each of
\begin{align}
  \vec{z}\tran \vec{y} & \leq \EE[ \vec{z} \tran \vec{y}] + 2 \const{c40} \sqrt{n} \cdot \sqrt{\ln(n\nc)} \\
   \vec{z} \tran \vec{y} &\geq   \EE[\vec{z}\tran\vec{y}] - 2 \const{c40} \sqrt{n} \cdot \sqrt{\ln(n\nc)},
\end{align} 
for some universal constant \const{c40}.
\end{proposition}

\begin{proof}
The upper-bound follows as
\begin{align}
    \vec{z}\tran \vec{y} &= \frac{1}{4} \left((\vec{z} + \vec{y}) \tran (\vec{z} + \vec{y}) - (\vec{z} - \vec{y}) \tran (\vec{z} - \vec{y})  \right) \\
    &\leq \EE[\vec{z}\tran \vec{y}] + \frac{K^2}{2\sqrt{c}} \sqrt{n} \cdot \sqrt{\ln n \nc},
\end{align}
with probability at least $(1 - 4/(n\nc))$, where we apply the Hanson-Wright inequality to each of the quadratic terms with $t = \frac{K^2}{\sqrt{c}} \sqrt{n} \sqrt{\ln( n\nc)}$ and use the fact that, letting $\Mbold =\In$, $\| \Mbold\|^2_F = n, \| \Mbold\|_{op} = 1$. The lower-bound can be obtained analogously, and an appropriate choice of $\newc\label{c40}$ completes the proof.
\end{proof}
From this, we can now prove Lemma \ref{lemma:zdiffybound}, restated below for convenience:
\zdiffybound*

\begin{proof}
We have
\begin{align}
     \za \tran \Ainv (\ya - \yb) &= \za \tran \left(\mubar \In+ \Ainvdel \right) (\ya - \yb) \\
    &= \mubar \za \tran (\ya - \yb) + \za \tran \Ainvdel (\ya - \yb) \\
    &= \mubar \za \tran (\ya - \yb) + \za \tran \Ainvdel (\ya^{oh} - \yb^{oh}).
\end{align}
We again simply bound
\eqn{
    | \za \tran \Ainvdel \dely | &\leq \|\zj \|_2 \| \Ainvdel \dely \|_2 \\
   &\leq  \| \Ainvdel \|_{op} \| \za \|_2 \| \dely \|_2\\
   &\leq \delmu \| \za \|_2 \| \dely \|_2 \\
   &\le \delmu \cdot \sqrt{(1 + \delta) n} \cdot \sqrt{(1 + \delta) \frac{2n}{\nc}} \\
   &= \delmu (1+\delta)\sqrt{2} \frac{n}{\sqrt{k}},
}
with probability $(1 - 2e^{-n\delta^2} - 2e^{-\frac{2n\delta^2}{3k}})$, using chi-squared bounds for $\za$ (Lemma \ref{lemma:chisquared}) and Chernoff bounds for $\dely$ (Proposition \ref{proposition:Mfrobnormbound}).

With probability $(1 - 2e^{-n\delta^2} - 2e^{-\frac{2n\delta^2}{3k}})$, we get each of
\eqn{
    \za^T \Ainv (\ya - \yb) &\le \mubar \za^T (\ya - \yb) + \delmu (1+\delta)\sqrt{2} \frac{n}{\sqrt{k}} \\
    \za^T \Ainv (\ya - \yb) &\ge \mubar \za^T (\ya - \yb) - \delmu (1+\delta)\sqrt{2} \frac{n}{\sqrt{k}}.
}

By applying Proposition \ref{proposition:zybound} on the relevant terms, setting $\delta$ to be an arbitrary value in $(0, 1)$, and choosing an appropriate constant \newc\label{c50}, we obtain with probability $(1 - 5/(nk))$ each of
\eqn{
    \za^T \Ainv (\ya - \yb) &\le \mubar (\EX[\za^T \ya] - \EX[\za^T\yb]) + \const{c50}(\mubar \sqrt{n}\sqrt{\ln(nk)} + \delmu n/\sqrt{k}) \\
    \za^T \Ainv (\ya - \yb) &\ge \mubar (\EX[\za^T \ya] - \EX[\za^T\yb]) - \const{c50}(\mubar \sqrt{n}\sqrt{\ln(nk)} + \delmu n/\sqrt{k}).
}
The probability comes from the union bound $(1 - 2e^{-n\delta^2} - 2e^{-\frac{2n\delta^2}{3k}} - 4/(nk)) \ge 1-5/(nk)$ (for sufficiently large $n$).
\end{proof}

\subsection{Proof of Corollary \ref{corollary:zdiffyboundnoexp}}
We claim the following bound:
\begin{proposition} \label{proposition:expzaya}
Bounds on $\EE[\za \tran \ya]$.
\begin{align}
  \frac{1}{\sqrt{\pi \ln 2}}  \cdot \frac{n}{\nc} \cdot \sqrt{\ln \nc} \leq \EE[\za \tran \ya] \leq \sqrt{2}  \cdot \frac{n}{\nc} \cdot \sqrt{\ln \nc} \label{eqn:expzayabound}
\end{align}
\end{proposition}

\begin{proof}
\begin{align}
    \EE[\za \tran \ya] &=  \EE[\za \tran \yaoh] - \EE[\za \tran \frac{1}{c} \mathbf{1}] \\
    &= \EE[\za \tran \yaoh] \\
    &= n \left( \EE[\zai \yaohi | \yaohi = 1] P(\yaohi=1) + \EE[\zai \yaohi | \yaohi = 0] P(\yaohi=0) \right) \\
    &= \frac{n}{\nc} \EE[\zai | \yaohi = 1].
\end{align}
So the desired bound follows from the bounds in Lemma \ref{lemma:maxofkgaussians}.
\end{proof}

We can obtain a similar bound for when $\beta \ne \alpha$:
\begin{proposition} \label{proposition:expzayb}
    Bounds on $\EE[\za \tran \yb]$.
    \begin{align}
      -\sqrt{2}  \cdot \frac{n}{\nc} \cdot \frac{1}{k-1} \cdot \sqrt{\ln \nc} \leq  \EE[\za \tran \yb] \leq - \frac{1}{\sqrt{\pi \ln 2}}  \cdot \frac{n}{\nc}\cdot  \frac{1}{k-1} \cdot \sqrt{\ln \nc} \label{eqn:expzaybbound}
    \end{align}
\end{proposition}
\begin{proof}
Observe that,
\begin{align}
    \EE[\za \tran \yb] &=  \EE[\za \tran \yboh] - \EE[\za \tran \frac{1}{\nc} \mathbf{1}] \\
    &= \EE[\za \tran \yboh] - \frac{1}{\nc} \EE[\za]\tran \mathbf{1} \\
    &= \EE[\za \tran \yboh] \\
    &= \sum_i \EE[ \zai \ybohi] \\
    &= n \left( \EE[\zai \ybohi | \ybohi = 1] P(\ybohi=1) + \EE[\zai \ybohi | \ybohi = 0] P(\ybohi=1) \right) \\
    &= \frac{n}{\nc} \EE[\zai | \ybohi = 1] \label{eqn:expectedzayb}
\end{align}

Now, observe that
\eqn{
    && \EE[\zai | \yaohi = 0] &= \sum_{\beta \ne \alpha} \EE[\zai | \ybohi = 1] \Pr(y_{\beta,i} = 1 \mid \yaohi = 0) \\
    &&&= \frac{1}{\nc-1} \sum_{\beta \ne \alpha} \EE[\zai | \ybohi = 1] \\
    &&&= \EE[\zai | \ybohi = 1]
}
for a particular $\beta$, by symmetry over the possible $\beta$.

Next we bound $\EE[\zai | \yaohi = 0]$ as follows:
\eqn{
    && \EE[\zai | \yaohi = 1] P(\yaohi=1) \nonumber \\
    &+& \EE[\zai | \yaohi = 0]P(\yaohi=0) &=   \EE[\zai] = 0\\
    \thus  \EE[\zai | \yaohi = 0] \frac{\nc - 1}{\nc} &= -  \EE[\zai | \yaohi = 1]\frac{1}{\nc} \\
    \thus  \EE[\zai | \yaohi = 0] &= -\EE[\zai | \yaohi = 1]\frac{1}{\nc - 1}
}
Thus, substituting in the results from Lemma \ref{lemma:maxofkgaussians}, and plugging back into \eqref{eqn:expectedzayb}, we obtain
\begin{align}
    -\sqrt{2}  \cdot \frac{n}{\nc} \cdot \frac{1}{k-1} \cdot \sqrt{\ln \nc} \leq  \EE[\za \tran \yb] \leq - \frac{1}{\sqrt{\pi \ln 2}}  \cdot \frac{n}{\nc}\cdot  \frac{1}{k-1} \cdot \sqrt{\ln \nc},
\end{align}
the desired result.
\end{proof}

We can now prove Corollary \ref{corollary:zdiffyboundnoexp}, which we restate below for convenience:
\zdiffyboundnoexp*
\begin{proof}
This follows by substituting the lower bound from \eqref{eqn:expzayabound} in Proposition \ref{proposition:expzaya} and the upper bound from \eqref{eqn:expzaybbound} in Proposition \ref{proposition:expzayb} into~\eqref{eqn:zdiffyboundlower} from Lemma \ref{lemma:zdiffybound}, making an appropriate choice for \newc\label{c60}.
\end{proof}

\section{Misclassification Events: Proof of Lemmas used in Theorem \ref{theorem:regimes}}

\label{app:misclassification_events}
With the previous section's utility bounds that allow us to deal with multiclass training data in hand, we are in a position to establish all the lemmas that we need to analyze misclassification.

\subsection{Proof of Lemma \ref{lemma:marginbound}: Lower bound on $\min_\beta(X_\a - X_\b)$}
With these bounds in hand, we can look at each misclassification event in turn. The first event to consider is if the best competing feature is unusually close to the true (maximum) feature.

\marginbound*


\begin{proof}

The following result from \citep{Marginproof} whose proof we reproduce here\footnote{We do this for the convenience of the reviewers since the source we are citing is a URL online. We believe that this is in the spirit of fair use.}, enables us to bound the closest feature margin as:
\eqn{
    \Pr(\min_\beta \left( x_{test}[\a] - x_{test}[\b] \right) > \theta / \sqrt{2 \ln(k)}) \ge \const{c150} e^{-\theta} \label{eqn:citedmarginbound},
}
for some universal positive constant \newc\label{c150}, for sufficiently large $k$. Thus, by selecting a constant $\theta$ such that $\const{c150} e^{-\theta} = 1 - \eps$ and choosing a sufficiently large $k$, we have that with probability $(1 - \eps)$:
\eqn{
    \min_\beta(x_{test}[\a] - x_{test}[\b]) \ge \theta / \sqrt{2 \ln k}.
}

The proof \citep{Marginproof} is reproduced below, with slight adaptations to match our use-case: Let $\b$ be the index of the largest competing feature to $x_{test}[\a]$. 
Then,
their joint PDF becomes
\eqn{
    f(x_{test}[\b], x_{test}[\a]) = k(k-1)F(x_{test}[\b])^{k-2}f(x_{test}[\b])f(x_{test}[\a]) \mathbf{1}(x_{test}[\b] < x_{test}[\a]).
}
where $F$ and $f$ are the CDF and PDF of the standard Gaussian. Let
\eqn{
    x = \frac{\theta}{\sqrt{2\ln(k)}}. \label{eqn:xdefn}
}
Thus,
\eqn{
    && \Pr\left(x_{test}[\a] - x_{test}[\b] > x\right) = k(k-1)J \label{eqn:marginprob},
}
where $J$ is defined as
\eqn{
    && J &= \int_{-\infty}^\infty \int_x^\infty F(w)^{k-2}f(w) f(v + w) \, \dv \, \dw \\
    &&&= \int_{-\infty}^\infty F(w)^{k-2}f(w) \left(\int_x^\infty f(v + w) \, \dv \right) \, \dw \\
    &&&= \int_{-\infty}^\infty F(w)^{k-2}f(w) (1 - F(x + w)) \, \dw.
}
Substituting $u = F(w)$, we have that
\eqn{
    && J &= \int_0^1 u^{k-2} (1 - F(x + F^{-1}(u))) \, \du \\
    &&&= \int_0^{1 - \ln(k)^2/k} u^{k-2} (1 - F(x + F^{-1}(u))) \, \du + \int_{1 - \ln(k)^2/k}^{1 - 1/(k\ln(k))} u^{k-2} (1 - F(x + F^{-1}(u))) \, \du \nonumber \\ 
    &&& \quad + \int_{1 - 1/(k\ln(k))}^1 u^{k-2} (1 - F(x + F^{-1}(u))) \, \du \label{eqn:marginj},
}
splitting $[0, 1]$ into three intervals, and integrating separately over each one. Let the three integrals be $J_1$, $J_2$, and $J_3$.

We have that
\eqn{
    && J_1 &= \int_0^{1 - \ln(k)^2/k} u^{k-2} (1 - F(x + F^{-1}(u))) \, \du \\
    &&&\le \int_0^{1 - \ln(k)^2/k} u^{k-2} \, \du \\
    &&&\le (1 - \ln(k)^2/k)^{k-2} \\
    &&&\le \exp\left( -\frac{k-2}{k} \ln(k)^2 \right) \\
    &&&= o\left(\frac{1}{k^2}\right) \label{eqn:j1simplification}.
}

Similarly,
\eqn{
    && J_3 &= \int_{1 - 1/(k\ln(k))}^1 u^{k-2} (1 - F(x + F^{-1}(u))) \, \du \\
    &&&= \int_{1 - 1/(k\ln(k))}^1 u^{k-2} (1 - F(F^{-1}(u))) \, \du \\
    &&&\le \int_{1 - 1/(k\ln(k))}^1 u^{k-2} (1 - u) \, \du \\
    &&&\le \frac{1}{k\ln(k)} \int_{1-1/(k\ln(k))}^1 u^{k-2} \, \du \\
    &&&= o\left(\frac{1}{k^2}\right) \label{eqn:j3simplification}.
}

Finally, in the intermediate interval $u \in [1-\ln(k)^2/k, 1-1/(k\ln(k))]$, as $k \to \infty$, we see that $u \ge (1-\ln(k)^2/k) \to 1$, $x \to 0$, and $F^{-1}(u) \to \infty$, so for sufficiently large $k$,
\eqn{
    && 1 - F(x + F^{-1}(u)) &\simeq \frac{f(x + F^{-1}(u))}{x + F^{-1}(u)} \\
    &&&\simeq \frac{f(x + F^{-1}(u))}{F^{-1}(u)} \\
    &&&\simeq \frac{f(F^{-1}(u))}{F^{-1}(u)} e^{-xF^{-1}(u)} \\
    &&&\simeq (1 - F(F^{-1}(u))) e^{-xF^{-1}(u)} \\
    &&&\simeq (1 - u) e^{-xF^{-1}(u)} \label{eqn:1fxusimplification},
}
applying the well-known approximation for the Gaussian CCDF $1 - F(w) \approx f(w)/w$ for large $w$ (for example, see Eqn. 8.2.38 from \citet{gallager1968information}), and substituting in the Gaussian PDF.

Further, since $F^{-1}(u) \to \infty$, we have that
\eqn{
    && 1 - F(F^{-1}(u)) &\simeq \frac{f(F^{-1}(u))}{F^{-1}(u)} \\
    \thus 1 - u &\simeq \frac{e^{-(F^{-1}(u))^2/2}}{F^{-1}(u) \sqrt{2\pi}} \\
    &&&\simeq e^{-(F^{-1}(u))^2/(2 + o(1))} \\
    \thus F^{-1}(u) &\simeq \sqrt{-2\ln(1 - u)} \\
    &&&\simeq \sqrt{2\ln(k)} \label{eqn:1fusimplification},
}
where the last step follows from the bounds on $u$ in the intermediate interval.

Substituting the bounds from \eqref{eqn:1fxusimplification} and \eqref{eqn:1fusimplification} into the expression for $J_2$, and applying the definition of $x$ from \eqref{eqn:xdefn}, we have, 
\eqn{
    && J_2 &\simeq \int_{1 - \ln(k)^2/k}^{1 - 1/(k\ln(k))} u^{k-2} (1 - F(x + F^{-1}(u))) \, \du \\
    &&&\simeq \int_{1 - \ln(k)^2/k}^{1 - 1/(k\ln(k))} u^{k-2} (1-u)e^{-xF^{-1}(u)} \, \du \\
    &&&\simeq \int_{1 - \ln(k)^2/k}^{1 - 1/(k\ln(k))} u^{k-2} (1-u)e^{-(\theta/\sqrt{2\ln(k)})\sqrt{2\ln(k)}} \, \du \\
    &&&\simeq \int_{1 - \ln(k)^2/k}^{1 - 1/(k\ln(k))} u^{k-2} (1-u)e^{-\theta} \, \du \\
    &&&\simeq e^{-\theta} \left[\frac{u^{k-1}}{k-1} - \frac{u^{k-2}}{k-2}\right]_{1 - \ln(k)^2/k}^{1 - 1/(k\ln(k))} \\
    &&&\simeq \frac{e^{-\theta}}{(k-1)(k-2)} \label{eqn:j2simplification}.
}

Combining the terms from \eqref{eqn:j1simplification}, \eqref{eqn:j3simplification}, and \eqref{eqn:j2simplification}, and substituting back into \eqref{eqn:marginprob}, we see that
\eqn{
    && \Pr\left(x_{test}[\a] - x_{test}[\b] > x\right) &= k(k-1)J \\
    &&&= k(k-1)(J_1 + J_2 + J_3) \\
    &&&\simeq k(k-1)\left(\frac{e^{-\theta}}{(k-1)(k-2)} + o\left(\frac{1}{k^2}\right)\right) \\
    &&&\simeq e^{-\theta}.
}
Expressing this as a non-asymptotic lower-bound on the probability, holding for sufficiently large $k$, yields the cited result in \eqref{eqn:citedmarginbound}.


\end{proof}

\constantmarginbound*
\begin{proof}
Observe that,
\eqn{
    \min_{1 \le \beta \ne \alpha \le c_k}(x_{test}[\a] - x_{test}[\b]) \ge \min_{1 \le \beta \ne \gamma \le c_k} \abs{x_{test}[\beta] - x_{test}[\gamma]}.
}
In other words, rather than bounding the margin between the largest and second-largest features, we will lower-bound the absolute difference between
any pair of features.

Consider a particular $(\beta, \gamma)$ tuple. Observe that $x_{test}[\beta] - x_{test}[\gamma] \sim N(0, 2)$, since each feature is drawn independently from a standard Gaussian. For any $\epsilon'>0$,  we can upper-bound
\eqn{
    && \Pr\left(\abs{x_{test}[\beta] - x_{test}[\gamma]} \le \eps'\right) &\le \frac{\eps'}{\sqrt{\pi}}
}
 by taking the product of the maximum value of the Gaussian pdf and the width, $2\epsilon'$, of the region we are interested in. Taking the union bound across all $(\beta, \gamma)$ tuples, we find that
\eqn{
    \Pr\left(\min_{1 \le \beta \ne \gamma \le c_k} \abs{x_{test}[\beta] - x_{test}[\gamma]} \le \eps'\right) \le \frac{c_k^2\eps'}{\sqrt{\pi}}.
}
So for any given $\eps > 0$, we can choose $\eps' = \eps\sqrt{\pi}/c_k^2$, and have that
\eqn{
    \Pr\left(\min_{1 \le \beta \ne \gamma \le c_k} \abs{x_{test}[\beta] - x_{test}[\gamma]} \ge \eps' \right) \ge 1 - \eps.
}
\end{proof}
\subsection{Lower bound on $\frac{\lambda \hhat_{\a,\b}[\a]}{\max_\b \cn_{\a,\b}}$}
Next, we will find a lower bound for survival-contamination ratio within the regime with low survival variance. 


\featuredifferencebound*
\begin{proof}
Using Corollary \ref{corollary:zdiffyboundnoexp}, we lower bound $\hhat_{\a,\b}[\a]$ with probability at least $\left(1 - 5/(nk)\right)$ as
\eqn{
    \hhat_{\a,\b}[\a] &= \lambda_\a^{-1/2}(\hat{f}_\a[\a] - \hat{f}_\b[\a]) \\
    &= \za \tran \Ainv \ya - \za \tran \Ainv \yb \\
    &\ge \const{c60}\mubar\frac{n}{k}\sqrt{\ln(k)}
    - \const{c50}(\mubar \sqrt{n} \sqrt{\ln(nk)} + \delmu \cdot n/\sqrt{k}).
}
Multiplying through by $\lambda$ gives the desired result.
\end{proof}

From the above result, under the scalings of our bi-level model we obtain:


\asymptoticfeaturedifferencebound*

\begin{proof}
Substituting our asymptotic scalings into the results from Lemma \ref{lemma:featuredifferencebound} and using the decay rate of $\mubar \asymp n^{-p}$ from Corollary \ref{corollary:delmuscaling} (which we can do since $1 < q + r < (p+1)/2$), we find that
\eqn{
    \lambda \hhat_{\a,\b}[\alpha] &\geq n^{p-q-r} \left(\const{c60}\mubar\frac{n}{k}\sqrt{\ln(k)}
    - \const{c50}(\mubar \sqrt{n} \sqrt{\ln(nk)} + \delmu \cdot n/\sqrt{k})\right) \\
    &= \const{c60} n^{1-q-r-t} \sqrt{\ln(k)} - \const{c50} n^{1/2-q-r}\sqrt{\ln(nk)} - \const{c50}n^{2-2q-2r-t/2} \\
    &\ge \const{c120} n^{\max(1-q-r-t, 2-2q-2r-t/2)} \sqrt{\ln(k)} \\
    &= \const{c120} n^{1-q-r-t + \max(0,1-q-r+t/2)} \sqrt{\ln(k)} \\
    &\ge \const{c120} n^{1-q-r-t} \sqrt{\ln(k)},
}
for an appropriately chosen universal constant \newc\label{c120} and sufficiently large $n$.
\end{proof}

Next we upper bound $\max_\b \cn_{\a,\b}$.
\contaminationbound*
\begin{proof}
For each $\b$ we have,
\begin{align}
    \cn_{\a,\b} &= \sqrt{\left(\sum_{j \notin \{\a, \b\}} \lambda_j^2 (\hhat_{\b,\a}[j])^2 \right)}
\end{align}

For $j \notin \{\a, \b\}$, by Lemma \ref{lemma:zdeltaybound},
\begin{align}
    \left|\hhat_{\b,\a}[j]\right| &= \left|\hhat_{\a,\b}[j]\right| \\
    &= \left|\fhat_j - \ghat_j\right|\\
    &= \left| \zj \tran \Ainv \ya - \zj \tran \Ainv \yb \right| \\
    &= \left| \zj \tran \Ainv (\ya -\yb) \right| \\
    &\le \const{c30}(\mubar \sqrt{\frac{n}{\nc}}\cdot \sqrt{\ln(nd\nc)} + \delmu \cdot n/\sqrt{k}),
\end{align}
with probability $1 - 7/(ndk)$.

So taking the union bound over all $d-2$ terms in the expression for the contamination, we can upper-bound it as
\eqn{
    CN_{\a,\b} &\le \const{c30}(\mubar \sqrt{\frac{n}{\nc}}\cdot \sqrt{\ln(nd\nc)} + \delmu \cdot n/\sqrt{k}) \sqrt{\sum \lambda_j^2},
}
with probability $\left(1 - 7/(nk)\right)$, the desired result.
\end{proof}

\asymptoticcontaminationbound*

\begin{proof}
Since $1 < q + r < (p+1)/2$, we can apply Corollary \ref{corollary:delmuscaling} to the result from Lemma \ref{lemma:contaminationbound} and substitute in the known scalings of various terms, to obtain
\eqn{
    && CN_{\a,\b} &\le \const{c30}(n^{1/2-t/2-p} \sqrt{\ln(ndk)} + \const{c80}n^{2-p-q-r-t/2}) (n^{p-q-r/2} + n^{p/2}) \\
    &&&\le \const{c130} n^{(1-t-p)/2 + \max(0,3/2-q-r) + \max(0,p/2-q-r/2)} \sqrt{\ln(ndk)},
}
for an appropriately chosen universal positive constant \newc\label{c130}.

\end{proof}
\subsection{Proof of Lemma \ref{lemma:differencesurvival}: Bounds on Survival Variance}
Finally, we look at the error event where a competing feature has unusually high survival relative to the true feature, so it is incorrectly selected. 


\differencesurvival*

\begin{proof}
We first consider the numerator of the LHS of \eqref{eqn:differencesurvivalbound}. By Lemma \ref{lemma:zdiffybound}, with probability at least $(1 - 5/(nk))$,
\eqn{
    \hhat_{\a,\b}[\a] &= \lambda_\a^{-1/2}(\hat{f}_\a[\a] - \hat{f}_\b[\a]) \\
    &= \za \tran \Ainv \ya - \za \tran \Ainv \yb \\
    &\leq \mubar(\EE[ \za\tran \ya] -\EE[ \za\tran \yb]) + \const{c50}(\mubar \sqrt{n} \sqrt{\ln(nk)} + \delmu \cdot n/\sqrt{k}).
}
Similarly, with probability at least $(1 - 5/(nk))$, 
\eqn{
    \hhat_{\b,\a}[\b] &= \lambda_\b^{-1/2}(\hat{f}_\b[\b] - \hat{f}_\a[\b]) \\
    &= \zb \tran \Ainv \yb - \zb \tran \Ainv \ya \\
    &\geq \mubar(\EE[ \zb\tran \yb] -\EE[ \zb\tran \ya]) - \const{c50}(\mubar \sqrt{n} \sqrt{\ln(nk)} + \delmu \cdot n/\sqrt{k}).
}

By symmetry,
\begin{align}
    \EE[ \zb\tran \yb] &= \EE[ \za \tran \ya] \\ 
     \EE[ \zb\tran \ya] &= \EE[ \za \tran \yb].
\end{align}

Thus with probability at least $\left(1 - 10/(nk) \right)$,
\eqn{
    && \hhat_{\a,\b}[\a] - \hhat_{\b,\a}[\b] &\le 2\const{c50}(\mubar \sqrt{n} \sqrt{\ln(nk)} + \delmu \cdot n/\sqrt{k}).
}

Using Corollary \ref{corollary:zdiffyboundnoexp} to lower-bound the denominator of the LHS of \eqref{eqn:differencesurvivalbound}, we obtain with probability at least $(1 - 15/(nk))$
\eqn{
    && \frac{\hhat_{\a,\b}[\a] - \hhat_{\b,\a}[\b]}{\hhat_{\a,\b}[\a]} &\le \frac{2\const{c50}(\mubar \sqrt{n} \sqrt{\ln(nk)} + \delmu \cdot n/\sqrt{k})}{\const{c60}\mubar\frac{n}{k}\sqrt{\ln(k)}
    - \const{c50}(\mubar \sqrt{n} \sqrt{\ln(nk)} + \delmu \cdot n/\sqrt{k})}.
}
\end{proof}

We can apply Corollary \ref{corollary:delmuscaling} to simplify our results from Lemma \ref{lemma:differencesurvival} in the asymptotic regime for the bi-level model.

\asymptoticdifferencesurvival*
\begin{proof}
Substituting, using Corollary \ref{corollary:delmuscaling}, in the regime where $1 < q + r < (p+ 1)/2$ and $t < 1/2$, we find that
\eqn{
    && \frac{\hhat_{\a,\b}[\a] - \hhat_{\b,\a}[\b]}{\hhat_{\a,\b}[\a]} &\le \frac{2\const{c50}(n^{1/2-p} \sqrt{\ln(nk)} + \const{c80} n^{2-p-q-r-t/2})}{\const{c60} n^{1-p-t} \sqrt{\ln(k)}
    - \const{c50}(n^{1/2-p} \sqrt{\ln(nk)} + \const{c80} n^{2-p-q-r-t/2})} \\
    &&&\le \frac{2\const{c50}}{\const{c60}} \cdot \frac{n^{1/2}\sqrt{\ln(nk)} + \const{c80} n^{2-q-r-t/2}}{n^{1-t} - (\const{c50}/\const{c60})n^{1/2}\sqrt{\ln(n)} + (\const{c50} \cdot \const{c80} /\const{c60})n^{2-q-r-t/2}} \\
    &&&\le \const{c140} \frac{n^{1/2}\sqrt{\ln(nk)} + n^{2-q-r-t/2}}{n^{1-t}} \\
    &&&\le \const{c140} n^{\max(t-1/2, t/2+1-q-r)} \sqrt{\ln(nk)},
}
for sufficiently large $n$ and an appropriate choice of positive constant \newc\label{c140}. Thus, if $\max(t-1/2, t/2+1-q-r) < 0$, our quantity of interest tends to zero at a polynomial rate as $n \to \infty$, completing the proof.
\end{proof}

\section{Conjectured Looseness of Bound}
\label{app:conjectured_loose}
In \eqref{eqn:cauchyschwartzlooseness} in the proof of Lemma \ref{lemma:zdeltaybound}, we upper bound $\zj \tran \Ainvdel \dely$ using the Cauchy-Schwarz inequality as
\begin{align}
    | \zj \tran \Ainvdel \dely | &\leq \|\zj \|_2 \| \Ainvdel \dely \|_2 \\
   &\leq  \| \Ainvdel \|_{op} \| \zj \|_2 \| \dely \|_2\\
   &\leq \delmu \| \zj \|_2 \| \dely \|_2.
\end{align}
This results in a high-probability bound of the order $\delmu n/\sqrt{k}$. 
Essentially this bound fears that $\Ainvdel$ can, in worst case, align $\zj$ and $\dely$ to be in the same direction. 
However, since there is only a weak dependence between $\Ainvdel$ and $\zj$ and $\dely$ this bound is likely overly cautious. 
We conjecture that this bound is loose by a factor $\sqrt{n}$. Why do we conjecture this? If we ignored the dependency of $\Ainvdel$ on $\zj$ and $\dely$ and blindly applied the Hanson-Wright inequality (with the $\Mbold$ matrix introduced as in Appendix~\ref{sec:introduceMbold} to leverage the fact that $\dely$ is mostly zeros) then we would obtain a high-probability upper bound of the form $ \delmu \sqrt{n/k}$ (ignoring the logarithmic factors).


Assuming this tighter conjectured bound holds and similarly assuming an analogously tighter bound for $\abs{\za \tran \Ainvdel \dely}$ in Appendix  \ref{app:proofzdiffybound} and following through with the rest of our analysis, we obtain the conjectured sufficient conditions for good generalization as in Equation \eqref{eq:conjectured_regimes} from Conjecture \ref{conjecture:regimes} for the regime $q+r > 1$. 

It turns out that whenever the survival/contamination ratio grows at a polynomial rate $n^v$ for $v > 0$ then the survival variation term also shrinks at a polynomial rate $n^{-u}$ for $u>0$. Thus ensuring the survival/contamination ratio is large enough (i.e. the number of classes is not too large relative to the level of favoring of potentially true features) is key to obtaining good generalization.

Although we focus on the regime $q+r>1$ in our work, our proof technique is also applicable to the regime $q+r < 1$, i.e where regression works and by grinding through the math for this setting we should be able to get sufficient conditions for good generalization here as well. 

Finally, we believe that we can adapt our analysis from the Proof of Theorem \ref{theorem:regimes} in Appendix \ref{app:proof} to write a set of sufficient conditions for poor generalization. The primary condition for this would be for the relevant survival/contamination ratio to go to zero. We conjecture that computing conditions on $p,q,r,t$ under which this occurs results in the converse result in the form of sufficient conditions for poor generalization present in Conjecture \ref{conjecture:regimes}. Intuitively, if the survival/contamination ratio goes to zero, then the contamination can with significant probability flip the sign of a comparison involving the score that should be winning --- this parallels the way that the converse is proved in \citet{binary:Muth20} for binary classification.



\section{Scaling parameters with the number of positive training examples per class}
\label{app:scaleexamples}
From our results in Figure~\ref{fig:bilevel_regimes}
we observed that as the number of classes $k$ increases (i.e.~larger values of $t$), the region  where multiclass classification generalizes well  shrinks. A justification for this is when the number of classes $k$ increases while the number of training points $n$ stays constant, we have fewer positive training examples from each class, and this makes the task harder. 

To see if the reduced number of positive training examples is indeed the dominant effect, we can explore what happens if we increase the number of total training points to compensate for this effect? Instead of scaling all parameters with the total number of training points, what happens if we scale them with the number of positive training examples per class?


Let $N = n^b$ be the new number of training points for some $b>1$, while rest of the parameters in the bi-level model scale as before. We have,
\eqn{
    && N &= n^b \\
    && d &= n^p = N^{p/b} \\
    && s &= n^r = N^{r/b} \\
    && a &= n^{-q} = N^{-q/b} \\
    && k &= c_kn^{-t} =c_k N^{-t/b}.
}
We can interpret this as our standard setup, albeit parameterized by $N$, rather than $n$. 
To keep the model well-defined we require the following:
\begin{itemize}
    \item $b < p$, to ensure we are still overparameterized;
    \item $r < b$, to ensure the number of favored features does not exceed the total number of training points;
    \item $q < p - r$ to ensure we are actually favoring the first $s$ features.
\end{itemize}

For this setup, Theorem~\ref{theorem:regimes} states that the probability of misclassification tends to zero if
\eqn{
    \frac{t}{b} &< \min\left( \frac{r}{b}, 1-\frac{r}{b}, \frac{p}{b}+1-2\left(\frac{q}{r}+\frac{r}{b}\right), \frac{p}{b}-2, \frac{2q}{b}+\frac{r}{b}-2 \right) \\
    \frac{q}{b} + \frac{r}{b} &> 1.
}
Rearranging, we obtain the condition
\eqn{
    && t &< \min(r, b-r, p+b-2(q+r),p-2b, 2q+r-2b) \\
    && q + r &> b.
}

To hold the number of training samples per class fixed we 
can set $b = t+1$, so the ratio $N/k$ becomes constant. Doing so, we obtain the following sufficient conditions for good generalization:
\eqn{
    && t  &< \min\left(r, \frac{p-2}{3}, \frac{2q+r-2}{3} \right)\\
    && 0 &< 1-r \\
    && 0 &< p+1-2(q+r)\\
    && t &< q + r - 1.
}
Additionally for the model to be well defined we require $t < p-1$. (The other conditions $r < t+1$ and $q < p-r$ for model to be well defined are automatically satisfied if the above conditions for good generalization are satisfied). 

If we assume Conjecture \ref{conjecture:regimes} then a set of sufficient conditions for good generalization is:
\eqn{
    && 0 &\le p+1-2(q+r) \\
    && r &< 1 \\
    && t &< r \\
    && t &< p - 1 .
}
The first two conditions must be satisfied for binary classification problem to generalize well and thus for multi-class classification to succeed in this setting we need to ensure binary classification succeeds.  
The condition $t<r$ arises because if we don't favor the features used in the comparison while assigning class labels then we have no hope of succeeding in overparameterized settings.
The condition $t < p-1$ ensures that the problem is overparameterized. 
If any of these conditions is not met then the probability of classification error will tend to 1. 

Figure~\ref{fig:scale_positive_regimes} visualizes the conjectured regimes for this alternative setup where the number of positive training examples per class is held fixed as we vary the number of classes for fixed values of $p$ and $q$. In the white region, our model is not well defined. Note that in subfigure (a), the limiting factor to the model being well defined is the inequality $r<1+t$ (we must have more training examples than favored features) while in subfigure (b), the limiting factor for the model being well defined in the right-hand boundary is the inequality $r < p-q$ (we must put a larger weight on the features we favor as compared to those that we do not favor). In subfigure(b) we see that the top boundary for the model being well defined is the inequality $t < p-1$ which is necessary for the problem to be overparameterized and support the existence of interpolating solutions. Further, the right-hand bound for good generalization in subfigure (a) corresponds to the inequality $r < 1$ while in subfigure (b) it corresponds to $p + 1 > 2(q+r)$. The left-hand boundary for good generalization in both figures is the inequality $t < r$, which reflects the fact that for MNI-based classification to succeed, all the features defining the classes must be favored.

\begin{figure*}[h!]
  \centering
  \includegraphics[width=0.80\columnwidth]{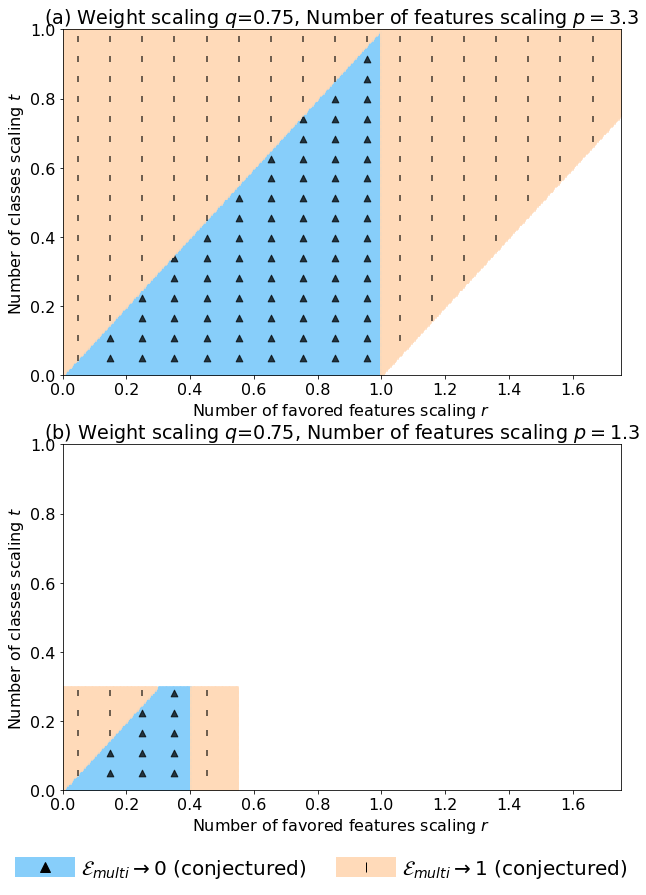}
  \caption{Visualization of the conjectured bi-level classification regimes when we scale everything with the number of positive training examples per class, instead of with the total number of training points.}
  \label{fig:scale_positive_regimes}
\end{figure*}

It is interesting to note that when we add more training points so as to increase the number of positive examples, we are effectively decreasing the level of overparameterization in the problem. We know from \citet{reg:nakkiran} that adding training data in a way that reduces overparameterization can sometimes make performance worse instead of better. However, in the deeply overparameterized setting of the bi-level models explored here, this effect is counteracted by the survival benefits of having more positive examples --- in effect, reducing the overall level of overparameterization reduces the shrinkage induced by the regularizing effect of overparameterization. This reduction in shrinkage compensates for the $\frac{1}{k}$ hit to survival induced by the larger number of classes.

\section{Additional related work} \label{app:comparisonToWang}
\subsection{Comparisons to \citet{multi_class_theory:Wang21}}
While our work has many similarities with \citet{multi_class_theory:Wang21} in terms of model and problem setting, there are some key differences. 

The first key difference is in how the training data is generated. 
In this paper, we assume the true label of a point is generated based on which of the first $k$ dimensions is the largest, while \citet{multi_class_theory:Wang21} consider a Gaussian mixture model and a multinomial logistic model where the true labels have some randomness even conditioned on the first $k$ dimensions. Like us, however, they also consider the case of orthogonal classes.

Second, we consider the asymptotic case where the number of classes, $k$, scales with the number of training points as $k = cn^t$ for some positive integer $c$ and non-negative real $t$.
The work in \citet{multi_class_theory:Wang21} considers only the finite classes setting i.e. $t=0$ in our model. 
The error analysis technique employed by us here in the form of a typicality-style argument featuring the feature margin (difference between the largest and second largest feature) is much tighter than the method employed in \citet{multi_class_theory:Wang21} and allows us to compute regimes where multiclass classification succeeds even when $t > 0$. 
A straight substitution into the analysis from \citet{multi_class_theory:Wang21} does not work since that analysis is too loose for this setting. 
Furthermore, in our expressions for survival and contamination (Lemmas \ref{lemma:featuredifferencebound} and \ref{lemma:contaminationbound}) we compute an exact dependence on $k$.\footnote{In particular, our analysis here brings out the fact that multiclass training data becomes less informative per training sample as the number of classes increases. This results in a $\frac{1}{k}$ scaling term in survival and a $\frac{1}{\sqrt{k}}$ scaling in contamination. It is this effect that makes it possible in some regimes for the contamination from other favored features to dominate --- whereas in the case of binary classification, it is always the contamination from unfavored features that dominates.}  
The expressions from \citet{multi_class_theory:Wang21} don't compute this exact dependence because it is not required for their purposes. 
By using our novel analysis technique we are able to elucidate the challenges posed by fewer positive training examples per class in the multiclass setting and provide sufficient conditions for generalization when number of classes scales with the number of training points.

An equivalence between the solution obtained by minimum-$\ell_2$-norm interpolation on the adjusted zero-mean one-hot encoded labels that we perform in our approach \eqref{eq:interpolateadjustedonehot} and the solution obtained by other training methods has been established in \citep{multi_class_theory:Wang21}. 
In particular the minimum-norm interpolating solution is typically identical to the solution obtained via one-vs-all SVM and multi-class SVM (and thus gradient descent on cross-entropy loss due to its implicit bias \citep{implicit_bias:ji2019, implicit_bias:srebro}, under sufficient overparameterization. From \citet{multi_class_theory:Wang21}, the sufficient conditions for the equivalence of solutions are,
\begin{align}
    \frac{\sum_{j=1}^{n} \lambda_j}{\lambda_1} &> C_1 k^2 n \ln(kn), \\
    \frac{(\sum_{j=1}^{n} \lambda_j)^2}{\sum_{j=1}^{n}\lambda_j^2} &> C_2(\ln(kn) + n),
\end{align}
where $C_1, C_2$ are positive constants. 
Under our bi-level model (Definition~\ref{def:bilevel}) these conditions translate to:
\begin{align}
    q+r &> 2t +1,\\
    2p - \max(2p-2q-r,p) &> 1,
\end{align}
which can be rearranged to give us the condition in \eqref{eqn:mnisvmequalregime}. Figure \ref{fig:svm_equivalence} from Section \ref{sec:conclusion} illustrates this regime, as well as how it relates to our results.



\subsection{Comparisons to \citet{binary:Muth20}}
The work in \citet{binary:Muth20} provides an analysis of the binary classification and regression problem with Gaussian features in the overparameterized regime and shows that binary classification is easier than regression by proving the existence of a regime in a bi-level model where binary classification generalizes well but regression does not.  
In this work we use a similar bi-level model and the signal-processing inspired concepts of survival and contamination in our proofs but the nature of the training data in the multiclass classification problem is the key challenge and complicates our analysis considerably. 
Since the true class labels are generated by comparing $k$ features, we no longer have independence of the class label $y$ with any of these features. This is relevant when we compute bounds on the the term $z_\a \tran \Ainv (y_\a - y_\b)$ an integral part of our survival quantity (Equations~\eqref{eqn:zdiffyboundupper},\eqref{eqn:zdiffyboundlower} from Lemma \ref{lemma:zdiffybound}), since the Hanson-Wright inequality is no longer applicable directly as was the case for the binary classification problem in prior work (Appendix D.3.1 of \citet{binary:Muth20}). 
Working through these challenges, we prove that the multiclass problem is fundamentally different from (and harder than) than the binary problem due to the effect of fewer informative samples (positive training examples) per class. 
In particular we show via dominant terms from Lemmas \ref{lemma:differencesurvival} and \ref{lemma:contaminationbound} that, as we increase the number of classes $k$, survival shrinks as $1/k$ while contamination shrinks only as $1/\sqrt{k}$. 
Thus the survival/contamination ratio which plays a key role in the expression for classification error decreases as $1/\sqrt{k}$ in the multiclass setting as we increase $k$. 
Thus, for good generalization we need to ensure number of classes is not too large in addition to having sufficient favoring of true features.




\end{document}